\documentclass[11pt, one column]{article}

\usepackage[english]{babel}
\usepackage[utf8x]{inputenc}
\usepackage[T1]{fontenc}

\usepackage[margin=1in]{geometry}

\usepackage{graphicx} 
\usepackage[colorlinks=False]{hyperref} 
\usepackage{amsmath}  
\usepackage{amsfonts} 
\usepackage{amssymb}  
\usepackage{amsthm}
\usepackage[shortlabels]{enumitem}
\usepackage{todonotes}
\usepackage{multirow}
\usepackage{float}
\usepackage{lipsum} % for filler text
\usepackage{pgfplots}
\usepackage{tikz}
\usetikzlibrary{calc}
\usepackage{authblk}

\newcommand{\C}{\mathcal{C}}
\newcommand{\F}{\mathcal{F}}
\newcommand{\K}{\mathcal{K}}

\newcommand{\aobj}[2]{a_{#1}^{#2}}
\newcommand{\alp}[2]{\alpha_{#1}({#2})}
\newcommand{\alpO}[2]{\alpha_{#1}(O_{#2})}
\newcommand{\cost}[2]{c_{#1}({#2})}
\newcommand{\costO}[2]{c_{#1}({O_{#2}})}
\newtheorem{theorem}{Theorem}[section]
\newtheorem{corollary}{Corollary}[theorem]
\newtheorem{lemma}[theorem]{Lemma}
\newtheorem{proposition}{Proposition}[section]

\title{Optimizing Multiple Simultaneous Objectives for Voting and Facility Location}
\author{Yue Han}
\author{Christopher Jerrett}
\author{Elliot Anshelevich}

\affil{Department of Computer Science, Rensselaer Polytechnic Institute}

\begin{document}
\maketitle
\begin{abstract} We study the classic facility location setting, where we are given $n$ clients and $m$ possible facility locations in some arbitrary metric space, and want to choose a location to build a facility. The exact same setting also arises in spatial social choice,
where voters are the clients and the goal is to choose a candidate or outcome, with the distance from a voter to an outcome representing the cost of this outcome for the voter (e.g., based on their ideological differences). Unlike most previous work, we do not focus on a single objective to optimize (e.g., the total distance from clients to the facility, or the maximum distance, etc.), but instead attempt to optimize several different objectives {\em simultaneously}. More specifically, we consider the $l$-{\em centrum} family of objectives, which includes the total distance, max distance, and many others. We present tight bounds on how well any pair of such objectives (e.g., max and sum) can be simultaneously approximated compared to their optimum outcomes. In particular, we show that for any such pair of objectives, it is always possible to choose an outcome which simultaneously approximates both objectives within a factor of $1+\sqrt{2}$, and give a precise characterization of how this factor improves as the two objectives being optimized become more similar. For $q>2$ different centrum objectives, we show that it is always possible to approximate all $q$ of these objectives within a small constant, and that this constant approaches 3 as $q\rightarrow \infty$. Our results show that when optimizing only a few simultaneous objectives, it is always possible to form an outcome which is a significantly better than 3 approximation for all of these objectives.% We conclude with some results for important special cases. 
\end{abstract}

\section{Introduction}
\label{section:intro}

When working on optimization problems, it is often difficult to pick one single objective to optimize: in most real applications different parties care about many different objectives at the same time. For example, consider the classic setting when we are given $n$ clients and $m$ possible facility locations in some metric space, and want to choose a location to build a facility. Note that while the setting we consider is for facility location problems, the exact same setting also arises in spatial social choice (see e.g., \cite{merrill1999unified,enelow1984spatial,anshelevich2021distortion}),
where voters are the clients and the goal is to choose a candidate or outcome located in some metric space, where the distance from a voter to an outcome represents the cost of this outcome for the voter (e.g., based on their ideological differences). When choosing where to build a facility (or which candidate to select) for the public good (e.g., where to build a new post office, supermarket, etc.), we may care about minimizing the average distance from users to the chosen facility (a utilitarian measure), or the maximum distance (an egalitarian measure), or many other measures of fairness or happiness. Focusing on just a single measure may not be useful for actual policy makers, who often want to satisfy multiple objectives simultaneously and in fact refuse to commit themselves to a single one, as many objectives have their own unique merits. In this paper we instead attempt to simultaneously minimize multiple objectives. %Not surprisingly, it is extremely rare that a single outcome fully optimizes multiple objectives. Because of this, we focus on obtaining outcomes that can simultaneously approximate multiple objectives within a small constant factor compared to the optimal solutions for these objectives.
For example, what if we care about {\em both} the average and the maximum distance to the chosen facility, and not just about some linear combination of the two? What if we want to choose a facility so that it is close to optimum in terms of the average distance from the users, and {\em at the same time} is also close to optimum in terms of the maximum distance? Is this even possible to do?

More specifically, we consider \emph{l-centrum} problems \cite{slater1978centers, tamir2001k, peeters1998some}, where we are given a set of possible facilities $\F$ and a set of $n$ clients $\C$ in an arbitrary metric space with distance function $d$. For each client $i \in \C$ there is a cost $d(i, j)$ if we choose to build facility $j \in \F$. Then the goal is to pick one facility from $\F$ such that it minimizes the sum of the $l$ most expensive costs induced by the choice of facility location. 
% ---
Such problems generalize minimizing the total client cost ($l=n$), as well as the maximum client cost ($l=1$). The latter may be considered a measure which is more fair to all the clients (since it makes sure that {\em all} clients have small cost, not just on average), but would have the drawback that a solution where all except a single client have low cost would be considered the same as a solution where they all have high cost, as long as the maximum cost stays the same. Because of this, some may argue that an objective where we consider only the costs of the worst 10 percent of the clients may be better. In this work, we side-step questions about which objective is best entirely. Since each of the \emph{l-centrum} objectives has its own advantages, our goal is to simultaneously approximate multiple such objectives. This idea of simultaneously approximating \emph{l-centrum} problems as a method of creating ``fair" outcomes was previously discussed in \cite{kumar2006fairness, goel2006simultaneous, goel2018relating}, and was adapted from the idea of \emph{approximate majorization} in \cite{bhargava2001using}.
 
Note that our approach is very different from combining several objectives into a single one (e.g., by taking a weighted sum); we instead want to make sure that the chosen outcome is good with respect to each objective we are interested in simultaneously. 
%There exist other kinds of objectives such as linear combinations of the \emph{l-centrum} objectives but we consider this specific set of problems partly because by simultaneously optimizing their objectives, we can achieve a notion of fairness. 
% ---
%This means that considering all or multiple of such objectives indeed has some advantages over considering only one of them.
%More formally, we define the cost vector induced by choosing facility $j \in \F$ to be $c_j = (c_{1j}, c_{2j}, \cdots, c_{nj})$ such that $c_{1j} \geq c_{2j} \geq \cdots \geq c_{nj}$ and a weight vector $w^l = (w_1, w_2, \cdots, w_n)$ such that $w_i = 1$ when $i \leq l$ and $w_i = 0$ otherwise. Then for $1 \leq l \leq n$, we define the cost function for choosing facility $j$ to be $\cost{l}{j} = w^l c_j$, then we want to minimize $\cost{l}{j}$ with fixed $l$ value, denote the optimal facility location of such objective by $O_l$. Assume we have $q$ such objectives such that $l \in \K = \{k_1, k_2, \cdots, k_q\}$ with distinct $k_i$ values. Then denote the set of their cost function by $\sS = \{c_{k_1}, c_{k_2}, \cdots, c_{k_q}\}$.
More formally, for $1 \leq l \leq n$, we define the cost function for choosing facility $A \in \F$ to be $\cost{l}{A}$, which is the sum of the top $l$ distances from $A$ to each client in $\C$. The $l$-{\em centrum} problem asks to minimize $\cost{l}{A}$ with a fixed $l$ value; denote the optimal facility location for this objective by $O_l$. Now suppose we have $q$ such objectives that we want to optimize, such that $l \in \K = \{k_1, k_2, \cdots, k_q\}$. We then say a facility $A \in \F$ is a simultaneous $\alpha$-approximation for all of the $q$ objectives iff $\cost{l}{A} \leq \alpha \cdot \costO{l}{l}$ for all $l \in \K$. %Note that while the setting we consider is for facility location problems, they can also be interpreted as spacial social choice problems (see e.g., \cite{merrill1999unified,enelow1984spatial,anshelevich2021distortion}), where voters are the clients and the candidates are facilities with the distance information corresponding to the cost for the voter of that candidate being elected (e.g., based on their ideological differences).

\subsection{Our Contributions}

We first consider the setting where we attempt to optimize two objectives. These objectives could, for example, be minimizing the maximum distance $\max_i d(i,j)$ and the total distance $\sum_i d(i,j)$. Or more generally they can be two arbitrary centrum objectives with one objective being \emph{k-centrum} and the other being \emph{p-centrum} with $k \leq p$. We prove that for any such pair of objectives, it is always possible to choose an outcome $A\in \F$ which simultaneously approximates both objectives within a factor of $1+\sqrt{2}$. In fact, we provide a tight upper bound for how well any such pair of objectives can be approximated at the same time, as shown in Figure \ref{fig:2_ub}. Our results show that when two people disagree on which objective is the best to optimize, they can both be made relatively happy in our setting. 
%Maybe remove? To prove this, we show that in every instance, either the best outcome for $c_p$ is also a pretty good outcome for $c_k$, or the best outcome for $c_p$ also has a pretty good $c_k$ value.

\begin{figure}[t]
  \begin{center}
      \begin{tikzpicture}
        \begin{axis}[
            xmin = 1, xmax = 20,
            ymin = 1, ymax = 3.0,
            restrict y to domain=1:3,
            xlabel     = $\frac{p}{k}$,
            ylabel     = $f$,
            clip       = false,
            ylabel style={rotate=-90},
            axis lines*=left
            ],
            \addplot[
            domain = 1:4,
            samples = 500,
            smooth,
            thick,
            blue,
            ] {sqrt(x)};
            \addplot[
            domain = 4:20,
            samples = 1000,
            smooth,
            thick,
            blue,
            ] {1 - 1/x + sqrt(1/x^2 - 2/x + 2)};
            \addplot[
            domain = 1:20,
            samples = 1000,
            smooth,
            dashed,
            orange,
            thick,
            ] {1 + sqrt(2)};
            \node[font=\tiny,pin= above:{$f(4) = 2$}] at (axis cs:4,2) {};
            \node[font=\tiny,pin= above right:{$f(1) = 1$}] at (axis cs:1,1) {};
            \node[font=\tiny,pin= above:{$1 + \sqrt{2}$}] at (axis cs:10,2.414) {};
        \end{axis}
    \end{tikzpicture}
  \caption{Plot of the tight upper bound of the simultaneous approximation ratio for $c_k$ and $c_p$ (denoted as function $f$) with respect to $\frac{p}{k}$. When $k$ and $p$ are similar this factor is small, and in fact when $p\leq 4k$ they can both be approximated within a factor of 2. As $\frac{p}{k}$ becomes larger, the worst-case approximation approaches $1+\sqrt{2}$.}
  \label{fig:2_ub}
  \end{center}
\end{figure}
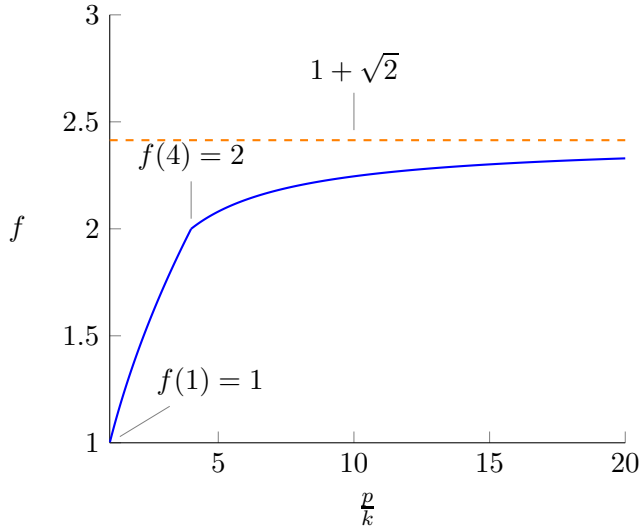

We then proceed to study optimizing more than two objectives at the same time. When optimizing $q$ different centrum objectives, we prove that it is always possible to approximate all $q$ of these objectives within a small constant; the plot of this upper bound is shown in Figure \ref{fig:mul_ub}. When $q=2$ this factor coincides with our above bound for 2 objectives; as the number of objectives grows this value approaches 3. Thus our results show, that when optimizing only a few simultaneous objectives, it is always possible to form an outcome which is a significantly better than 3 approximation for all of these objectives.\footnote{If the goal is to approximate {\em all} the centrum objectives at the same time, then \cite{kumar2006fairness} provided a 4-approximation for this, and the results in \cite{goel2018relating,gkatzelis2020resolving} imply that this can be improved to 3. %We also provide a factor of 3 for approximating all the centrum objectives at the same time (i.e., $q=n$), but our algorithm for choosing the outcome which forms this approximation is far simpler than the one in \cite{gkatzelis2020resolving}: we prove that simply choosing the optimum for $c_n$ already fulfills this requirement.
}

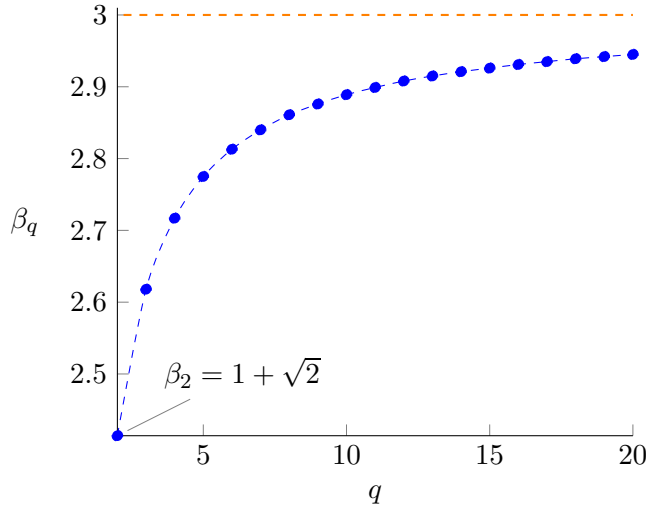
\begin{figure}[t]
  \begin{center}
  \begin{tikzpicture}
    \begin{axis}[
        xmin = 2, xmax = 20,
        ymin = 2.414, ymax = 3.01,
        restrict y to domain=1:3,
        xlabel     = $q$,
        ylabel     = $\beta_q$,
        % clip       = false,
        ylabel style={rotate=-90},
        axis lines*=left
        ],
        \addplot[
            smooth,
            blue,
            mark=*,
            dashed
        ] file[] {plot2data.dat};
        \addplot[
            domain = 0:20,
            samples = 1000,
            smooth,
            dashed,
            orange,
            thick
        ] {3};
        \node[font=\tiny,pin= above right:{$\beta_2 = 1 + \sqrt{2}$}] at (axis cs:2,2.414) {};
    \end{axis}
\end{tikzpicture}
  \caption{Plot of upper bound of the simultaneous approximation ratio $\beta_q$ for $q$ different centrum objectives.
  %for $q = 2, 3, \cdots, 10$. [Maybe remove? This bound is provably tight for $q=2,3$.]
  }
  \label{fig:mul_ub}
  \end{center}
\end{figure}

%We improve the 4 global approximation when $\K = \{1,2,\dots, n\}$ as shown \cite{kumar2006fairness} by showing that picking the optimal for the \emph{n-centrum} objective would result in a 3 global approximation. 
%For two objectives, we obtain a tight upper bound for $\alpha$ as shown in Figure \ref{fig:2_ub}, with one objective being \emph{k-centrum} and the other being \emph{p-centrum} with $k \leq p$. 
%We then show the upper bound for arbitrary $q$ objectives, a plot of the upper bound with respect to the value of $q$ is shown in Figure \ref{fig:mul_ub}. We also proved that the bound is tight when $q = 2, 3$. 
%Finally, we consider the important special case when facilities can be placed at any client location (i.e., $\F=\C$, which for the social choice setting corresponds to all the voters also being possible candidates) and show better bounds for optimizing two simultaneous objectives: it is always possible to choose an outcome which is a $2$-approximation for both objectives.
%that there must exist a 2 global approximation for two objectives. 

Finally, in Section \ref{section:same2} we discuss the important special case when facilities can be placed at any client location, i.e., $\F=\C$, which for the social choice setting corresponds to all the voters also being possible candidates.
%and show better bounds for optimizing two simultaneous objectives: it is always possible to choose an outcome which is a $2$-approximation for both objectives.

\subsection{Related Work}

% other objectives
Facility location, as well as spatial voting problems, are a huge area with far too much existing work to survey here; many variants have been studied with many different objectives (see for example \cite{chan2021mechanism, farahani2010multiple} and the references therein). 
As discussed above, we want to simultaneously approximate multiple objectives. However, there are many other different approaches when considering multiple objectives, with a common one being converting multiple objectives into one objective and optimizing the new objective, as in \cite{ehrgott2000survey, farahani2010multiple}. As discussed in Section 5.1 in \cite{ehrgott2000survey}, the most commonly used conversion is the weighted sum of the objectives (such as in, e.g., \cite{alamdari2017bicriteria, mcginnis1978single, ohsawa1999geometrical}), but of course there are many ways to combine several objectives. Since different people have different opinions and priorities, it is usually impossible to pick one combination that can make everyone satisfied.
% 
%For example, suppose we consider the weighted sum of minimax (minimize maximum cost) objective and minisum (minimize sum of costs) objective: some may desire equal weights for both objectives, but others would want to put a higher weight for the minimax objective. Because of the disagreement on priorities, we cannot find weights that could meet everyone's desire. 
% Even if we manage to find the desired weights, it is often the case that such method results in a fixed objective that exists an optimal solution and the goal would be to find efficient algorithms to find or approximate such optimal solution \cite{mcginnis1978single, ohsawa1999geometrical}. 
In addition, a good outcome for the new (combined) objective does not directly imply it is also good with respect to each of the original objectives, and in fact simultaneously approximating several objectives may be impossible even if the combined objective is approximable. 
%In fact, for general facility location problems, even with two facilities the simultaneous approximation ratio for both minimax and minisum can be unbounded as shown in \cite{alamdari2017bicriteria, goel2006simultaneous}. 
On the other hand, a simultaneous $\alpha$-approximation for all the objectives {\em does} imply an $\alpha$-approximation of any convex combinations of the objectives, and thus forms a strictly stronger result.

% other objectives -- (2)
In addition to forming a weighted sum of two objectives, 
%, aside from the aforementioned notions for approximation ratio, in the case of combining multiple objectives into one, 
% instead of approximating the optimal solution for the new objective,
\cite{alamdari2017bicriteria} looked at a different notion of approximation for multiple objectives. They considered approximating the original objectives with respect to a specific feasible solution instead of their respective optimal solutions, and gave a polynomial time (4,8)-approximation algorithm for minimax and minisum for the special case when $\F = \C$. A simultaneous $\alpha$-approximation for both minimax and minisum (which is what we form for $\alpha=1+\sqrt{2}$) would also imply a $(\alpha, \alpha)$-approximation in their notion of approximation. Note that the result in \cite{alamdari2017bicriteria} applies to selecting more than one facility, however, while in that setting we know it is impossible to form a bounded approximation ratio for both minimax and minisum simultaneously \cite{alamdari2017bicriteria, goel2006simultaneous}. Because of this, similarly to much of the work on this topic (see below), we focus on selecting a single facility while optimizing multiple simultaneous objectives.
When considering multiple objectives, another approach is to consider efficient algorithms to find the Pareto optimal set\footnote{Such set contains Pareto optimal solutions, or efficient solutions such that in any of these solutions, none of the objectives can be improved without simultaneously worsening any other objectives \cite{ehrgott2000survey, farahani2010multiple}.} as discussed in \cite{ehrgott2000survey}, with specific examples such as \cite{nickel2005multicriteria, roostapour2015deterministic} for placing one facility. However, this approach generally does not consider how good those solutions are in comparison to each of the objectives, nor if there exists one that is good for all the objectives. 

% why l-centrum
Now that we have discussed different approaches for multiple objective optimization, we want to consider a set of objectives that is appropriate for our setting. Two of the most commonly studied objectives for facility location problems are minimax (minimize maximum distance) and minisum (minimize sum of distances), see the survey \cite{chan2021mechanism}. In this work we study a more general version of these two objectives named the \emph{l-centrum} objectives. \emph{l-centrum} problems were first introduced by \cite{slater1978centers} and were later studied in other literature such as \cite{tamir2001k, peeters1998some}. This set of problems is also a subset of the \emph{ordered k-median} problems \cite{sornat2019approximation, chakrabarty2017interpolating, chakrabarty2019approximation, byrka2018constant, kalcsics2002algorithmic}. In fact, \emph{ordered k-median} objectives can be represented as convex combinations of the \emph{l-centrum} objectives as discussed in \cite{sornat2019approximation, chakrabarty2017interpolating, chakrabarty2019approximation}. This means that if we were to combine any \emph{l-centrum} objectives by convex combination into a new objective, \cite{chakrabarty2019approximation} gives a $(5+\epsilon)$ approximation for this new objective.

Note that the above work on \emph{l-centrum} problems only considered the approximation ratio for a single objective at a time, but our goal is to approximate multiple objectives simultaneously. %As discussed before, the simultaneous approximation ratio of multiple objectives for placing multiple facilities can be unbounded but for placing one facility, 
For this goal, \cite{kumar2006fairness} showed the existence of a simultaneous 4-approximation for all \emph{l-centrum} objectives, and the results in \cite{goel2018relating, gkatzelis2020resolving} imply a similar 3-approximation.
We provide a much simpler mechanism for obtaining this 3-approximation than the one in \cite{gkatzelis2020resolving}.\footnote{Although mostly concerned with different questions, the 3-approximation mechanism we use can also be obtained as an easy consequence of the results in \cite{goel2018relating}.}
However, our main focus is on improving the upper bound with respect to the {\em number} of \emph{l-centrum} objectives to be simultaneously approximated. We provide better approximations when the number of objectives to be approximated is small (instead of being {\em all} the \emph{l-centrum} objectives), with a much more detailed tight analysis for two objectives.  
%For this goal, \cite{kumar2006fairness} showed the existence of a simultaneous 4-approximation for $\K = \{1,2,\dots, n\}$ which also means that there exist a simultaneous 4-approximation for any subset of $\K$ with more than one element (and \cite{.} gives a 3-approximation). We instead focusing on improving the upper bound with respect to the number of \emph{l-centrum} objectives to be simultaneously approximated with a more detailed analysis on two objectives with tight upper bound.  

% mech design
Similar questions about facility location and voting have also been studied in mechanism design. %Another area that the facility location problems are considered a lot in is mechanism design. With respect to mechanism design, 
For instance, there has been a significant amount of work in mechanism design considering the approximation ratio for strategy-proof mechanisms for placing a single facility \cite{walsh2021strategy, alon2010strategyproof, feldman2016voting, tang2020mechanism}. Note that much of the previous work in this area studied only the 1D real line metric (e.g., \cite{walsh2021strategy, aziz2021strategyproof}), while we look at general arbitrary metric spaces. For simultaneously approximating two objectives, \cite{walsh2021strategy} showed that it is always possible to obtain a (3,3)-approximation for minimax and minisum for clients and facilities on a line. In addition, they also showed that no deterministic and strategy-proof mechanism can do better than 3-approximation for either of the two objectives. 

Finally, our work also applies to spatial voting instead of facility location, where voters and candidates are located in a metric space, and the goal is to choose a candidate which minimizes some objectives over the voters \cite{merrill1999unified,enelow1984spatial}. Perhaps the most relevant work to ours in this space is the work on {\em distortion}, where instead of knowing the voters' exact locations, each voter only provides their ordinal preferences for the candidates (see the survey \cite{anshelevich2021distortion}). As part of this work,  \cite{anshelevich2017randomized,feldman2016voting} showed that Random Dictatorship has an approximation ratio of 3 in a general metric space for minisum. More generally, the results in \cite{gkatzelis2020resolving} imply a simultaneous 3-approximation for all $l$-centrum objectives, by choosing a candidate using a somewhat complex, but deterministic, mechanism.
%In fact, it is a (3,3)-approximation for minimax and minisum as proven in \cite{tang2020mechanism}. 
One of our goals is to improve this upper bound of 3 for simultaneously approximating multiple objectives. Because of this, just as in most work on distortion, our main focus is not on strategyproofness. Instead we study how well multiple objectives can be approximated simultaneously, even if we are given all the necessary information.

\section{Preliminaries and Notation}
\label{section:prelim}
Consider the facility location problem where we are given the set of client locations $\C$ of size $n$, and the set of possible facility locations $\F$ in a metric space $d$. We want to select a location $A\in\F$ to place a facility such that the placement would simultaneously minimize some set of objectives. The kind of objectives that we are particularly interested in is the summation of the top $k$ distances from the clients to the chosen facility location. More formally, suppose we order the $n$ clients $a^1, a^2, a^3, \ldots, a^n$ so that
$$d(a^1,A)\geq d(a^2,A)\geq \ldots\geq d(a^n,A).$$

%we will first order the clients by their distances from $A$ in decreasing order. Suppose there are $n$ clients, then we denote them by
       % $$a^1, a^2, a^3, ... a^n$$
    %Where $a^1$ is the client that is the furthest away from $A$ and $a^n$ is the closest. Or, we can say that 
    %	$$a^1 = \operatorname{argmax}_{i \in \C}  d(i, A)$$
    %	$$a^n = \operatorname{argmin}_{i \in \C}  d(i, A)$$
Then, define the $k$-centrum objective $c_k(A)$ to be the cost for choosing facility location $A$ when considering the top $k$ client-facility distances, i.e., 
    	$$c_k(A) = \sum_{i = 1}^{k} d(a^i, A).$$
 Our goal is to minimize $c_k(A)$ for multiple $k$ simultaneously. Denote the optimal facility location for the $k$-centrum objective by $O_k$. It will be useful to denote by $\aobj{k}{i}$ the client that is $i$'th farthest from $O_k$, i.e.,  
$$d(\aobj{k}{1},O_k)\geq d(\aobj{k}{2},O_k)\geq\ldots\geq d(\aobj{k}{n},O_k).$$

% check wording     
We are given a set of objectives to minimize, represented by a a set of distinct positive integers $\K = \{k_1, k_2, \cdots, k_q\}$, with each of its elements less than or equal to $n$. This means that we want to simultaneously minimize all objectives $c_{k_i}$ for $k_i\in \K$. We slightly abuse notation and refer to $\K$ as the set of objectives, and say that an objective $c_k$ is in $\K$ when $k\in \K$. However, in order to simultaneously minimize the objectives in $\K$, we would have to make some trade-offs such that the chosen facility location may not be the optimal location for some of the objectives.\footnote{Here note that we assume $\cost{1}{A} > 0$; this means that the cost function is always positive. Otherwise we should just choose the facility location $A$ such that $\cost{1}{A} = 0$. This indicates that all clients are at most 0 distance away from $A$, which means that all clients are at the same location as $A$ given they are located in a metric space. In other words, $A$ would be the optimal facility location for all objectives in $\K$.} We thus define the approximation ratio for choosing facility location $A$ with respect to objective $c_k$ as 
    	$$\alpha_k(A) = \frac{c_k(A)}{c_k(O_k)} \geq 1$$
     Therefore, by choosing facility location $A$, we would obtain a $(\alpha_{k_1}(A), \alpha_{k_2}(A), \cdots, \alpha_{k_q}(A))$ approximation for minimizing the set of objectives $\K$. As discussed in the Introduction, our goal is to establish that we can always choose some $A$ so that all these values are small simultaneously.
    	
\section{Simultaneously Approximating Two Objectives}
\label{section:2_obj}
% check wording
We will first consider the case where there are only two objectives. Let $|\C| = n, 1 \leq k < p \leq n$, we then wish to simultaneously minimize $c_k(A)$ and $c_p(A)$. Our goal is to find some $A \in \F$ such that both $\alpha_k(A)$ and $\alpha_p(A)$ are small. In fact, with this goal in mind, we can obtain the following result:\\

\begin{theorem}
For $1 \leq k < p \leq n$, given the optimal facility location $O_k$ that minimizes $c_k$ and the optimal facility location $O_p$ that minimizes $c_p$, we have that $\alpha_k(O_p) \leq \frac{1}{\alpha_p(O_k)} + 2$.
\label{theorem:2_gen}
\end{theorem}

The above theorem indicates that by picking either $O_k$ or $O_p$, the values of $\alpha_k(O_p)$ and $\alpha_p(O_k)$ cannot be simultaneously large. In other words, either setting $A$ to be $O_k$ would ensure that both $\alpha_k(A)$ and $\alpha_p(A)$ are small or setting $A$ to be $O_p$ would. Now, in order to prove Theorem \ref{theorem:2_gen}, we will first show some simple lemmas.\\

\begin{lemma}
For any $k$, $p$ and any $A \in \F$, we have that $\sum_{i=1}^{p} d(\aobj{k}{i}, A) \leq \cost{p}{A}$.

\begin{proof}	
Here note that $\cost{p}{A} = \sum_{i=1}^{p} d(a^i, A)$, where $a^1, a^2, \cdots, a^p$ are the top $p$ clients furthest away from $A$ by definition. This means that when given any $p$ distinct client locations $b_1, b_2,\cdots, b_p$, we would have $\cost{p}{A} \geq \sum_{i=1}^{p} d(b_i, A)$. Therefore, we can conclude that $\sum_{i=1}^{p} d(\aobj{k}{i}, A) \leq \cost{p}{A}$.
\end{proof}
\label{lemma:diff_sum}
\end{lemma}

\begin{lemma}
For $1 \leq k < p \leq n, A \in \F$, we have that $\cost{p}{A} \leq \frac{p}{k} \cdot \cost{k}{A}$.

\begin{proof}
	By definition, $\cost{p}{A} = \sum_{i=1}^{p} d(a^i, A)$ is the summation of the largest $p$ distances from any clients to facility location $A$ and $\cost{k}{A} = \sum_{i=1}^{k} d(a^i, A)$ is the summation of the largest $k$ distances from any clients to facility location $A$. However, since $k < p$, we must have $\frac{1}{p} \cdot \sum_{i=1}^{p} d(a^i, A) \leq \frac{1}{k} \cdot \sum_{i=1}^{k} d(a^i, A)$. Hence, we can conclude that for $1 \leq k < p \leq n$, $\cost{p}{A} \leq \frac{p}{k} \cdot \cost{k}{A}$.
\end{proof}
\label{lemma:k_cost}
\end{lemma}

%---
% Now, with the above lemmas, we can obtain the following theorem:\\
Now, with the above lemmas, we can then present the proof for Theorem \ref{theorem:2_gen} as follows:\\

\begin{proof}[Proof of Theorem \ref{theorem:2_gen}]
	First, note that all the clients and possible facility locations are located in metric space $d$, so we have
\[
\begin{aligned}
	c_k(O_p) &= \sum_{i=1}^{k} d(\aobj{p}{i}, O_p)\\
	&\leq \sum_{i=1}^{k} \left[d(\aobj{p}{i}, O_k) + d(O_k, O_p)\right]\\
	&= \sum_{i=1}^{k} d(\aobj{p}{i}, O_k) + k\cdot d(O_k, O_p)
\end{aligned}
\]
Then, by Lemma \ref{lemma:diff_sum} and the triangle inequality, we have 
\[
\begin{aligned}
	c_k(O_p) &\leq \costO{k}{k} + k\cdot d(O_k, O_p)\\
	&\leq \cost{k}{O_k} + \frac{k}{p} \left[ \sum_{i=1}^{p} d(\aobj{k}{i}, O_k) + \sum_{i=1}^{p} d(\aobj{k}{i}, O_p) \right]\\
	&\leq  \cost{k}{O_k} + \frac{k}{p} \left[ \cost{p}{O_k} + \cost{p}{O_p} \right]\\
	&= \cost{k}{O_k} + \frac{k}{p} \left[ \cost{p}{O_k} + \cost{p}{O_k} \cdot \frac{\cost{p}{O_p}}{\cost{p}{O_k}} \right]\\
	&= \cost{k}{O_k} + \frac{k}{p} \left[ \cost{p}{O_k} + \frac{\cost{p}{O_k}}{\alp{p}{O_k}} \right]\\
	&= \cost{k}{O_k} + \frac{k}{p} \left[ \left(1 + \frac{1}{\alp{p}{O_k}}\right)\cost{p}{O_k} \right]
\end{aligned}
\]
Since $k < p$, by Lemma \ref{lemma:k_cost} we can see that
\[
\begin{aligned}
	c_k(O_p) &\leq \costO{k}{k} + \frac{k}{p} \left[ \left(1 + \frac{1}{\alp{p}{O_k}}\right)\cdot \frac{p}{k} \costO{k}{k} \right]\\
	& = \costO{k}{k} + \left(1 + \frac{1}{\alp{p}{O_k}}\right)\costO{k}{k}\\
	& = \left(2 + \frac{1}{\alp{p}{O_k}}\right)\costO{k}{k}
\end{aligned}
\]
Now, divide both side by $\costO{k}{k}$, then
\[
\begin{aligned}
	\frac{c_k(O_p)}{\costO{k}{k}} &\leq  \left(2 + \frac{1}{\alpO{p}{k}}\right)\frac{\costO{k}{k}}{\costO{k}{k}}\\
	\alpO{k}{p} &\leq \frac{1}{\alpO{p}{k}} + 2,
\end{aligned}
\]
as desired. 
\end{proof}

Note that Theorem \ref{theorem:2_gen} immediately implies the following corollaries:\\

\begin{corollary}
For $1 \leq k < p \leq n$,
\begin{enumerate}
	% \item By choosing the optimal facility location $O_k$ that minimize $c_k$, we can obtain a $(1, \alpO{p}{k})$ approximation for simultaneously minimizing $c_k$ and $c_p$ and would get a $(\frac{1}{\alpO{p}{k}}+2, 1)$ approximation by choosing the optimal facility location $O_p$ that minimize $c_p$ instead. This means by choosing either $O_k$ or $O_p$, we cannot improve either $\alpO{p}{k}$ or $\alpO{k}{p}$ without making the other one worse by choosing the other optimal location except when $\alpO{p}{k} = 1$.
	\item By choosing the optimal facility location $O_p$ that minimizes $c_p$, we obtain a $(3, 1)$ approximation for simultaneously minimizing $c_k$ and $c_p$. 
	\item There always exists a facility location $A\in\F$ such that choosing $A$ would give a $1+\sqrt{2}$ approximation both for minimizing $c_k$ and minimizing $c_p$. In fact, we would either get a $(1, 1+\sqrt{2})$ approximation by choosing $O_k$ or a $(1+\sqrt{2}, 1)$ approximation by choosing $O_p$. In other words, at least one of $\alpha_k(O_p)$ or $\alpha_p(O_k)$ is always less than or equal to $1+\sqrt{2}$. 
\end{enumerate}

\begin{proof}
	We will first show  $\alpO{k}{p} \leq 3$. Note that $\alpO{p}{k} = \frac{\costO{p}{k}}{\costO{p}{p}}$ by definition where $O_p$ is the optimal facility location for minimizing $c_p$. This means that we must have $\costO{p}{k} \geq \costO{p}{p}$, so $\alpO{p}{k} = \frac{\costO{p}{k}}{\costO{p}{p}} \geq 1$. Then, by Theorem \ref{theorem:2_gen} we have that $\alpha_k(O_p) \leq \frac{1}{\alpha_p(O_k)} + 2 \leq 1 + 2 = 3$ as desired. 

Now we will show that $\max (\min (\alpha_k(O_p), \alpha_p(O_k))) \leq 1 + \sqrt{2}$. 
%To do this, first assume $\alp{k}{p} = \alp{p}{k} = \beta$, then,
%$$\alpha_k(A_p) \leq \frac{1}{\alpha_p(A_k)} + 2 = \frac{1}{\alpha_k(A_p)} + 2$$
%\[
%\begin{aligned}
%	\beta &\leq \frac{1}{\beta} + 2\\
%	\beta^2 & \leq 1 + 2\beta\\
%	\beta^2-2\beta + 1 &\leq 2\\
%	(\beta - 1)^2 &\leq 2
%\end{aligned}
%\]
%Since $\beta \geq 1$, we have that $\beta \in [1, 1 + \sqrt{2}]$. 
To do this, we will consider two cases. First, assume $\alpO{p}{k} > 1 + \sqrt{2}$, then we would have $\alpha_k(O_p) \leq \frac{1}{\alpha_p(O_k)} + 2 < \frac{1}{1 + \sqrt{2}} + 2 = 1 + \sqrt{2}$. Otherwise, we would have $\alpO{p}{k} \leq 1 + \sqrt{2}$. Therefore, we can conclude that $\max (\min (\alpha_k(O_p), \alpha_p(O_k))) \leq 1 + \sqrt{2}$. This result implies that one of $\alpha_k(O_p)$ and $\alpha_p(O_k)$ is less than or equal to $1+\sqrt{2}$, as desired.
\end{proof}
\label{coro:2_gen}
\end{corollary}

The above results show that it is always possible to approximate any pair of our objectives to within a factor of $1+\sqrt{2}$ of optimum. However, it is natural to think that there exists some relationship between this approximation factor, and how similar the objectives are. Naturally, as the difference between $k$ and $p$ becomes smaller, we would expect that both $\alpO{p}{k}$ and $\alpO{k}{p}$ would also become smaller. In order to reveal such a potential relationship, we will first make some observations:
\\
\begin{lemma}
	For $1 \leq k < p \leq n$, we have that $\costO{p}{p} - \costO{k}{p} = \left[ \frac{1}{\alpO{p}{k}} - \alpO{k}{p}\right]\costO{k}{k} + \frac{1}{\alpO{p}{k}}\left[ \costO{p}{k} - \costO{k}{k} \right]$.

\begin{proof}
	First, recall that $\costO{p}{p} = \frac{\costO{p}{k}}{\alpO{p}{k}}, \costO{k}{p} = \alpO{k}{p}\cdot \costO{k}{k}$, then we have that 
\[
\begin{aligned}
	\costO{p}{p} - \costO{k}{p} &= \frac{\costO{p}{k}}{\alpO{p}{k}} - \alpO{k}{p}\cdot \costO{k}{k}\\
	&= \frac{\costO{k}{k}}{\alpO{p}{k}} + \frac{\costO{p}{k} - \costO{k}{k}}{\alpO{p}{k}} - \alpO{k}{p}\cdot \costO{k}{k}\\
	&= \left[ \frac{1}{\alpO{p}{k}} - \alpO{k}{p}\right]\costO{k}{k} + \frac{1}{\alpO{p}{k}}\left[ \costO{p}{k} - \costO{k}{k} \right]
\end{aligned}
\]
\end{proof}
\label{lemma:kpsum}
\end{lemma} 

\begin{lemma}
	For $1 \leq k < p \leq n, A \in \F$, we have that $\frac{k}{p-k}\left[ \cost{p}{A} - \cost{k}{A} \right] \leq \cost{k}{A} $.

\begin{proof}
	Since $k < p$, we have that $p - k > 0$. Similar to the proof for Lemma \ref{lemma:k_cost}, note that 
	$$\cost{p}{A} - \cost{k}{A} = \sum_{i = k+1}^{p} d(\aobj{}{i}, A) \leq (p-k) d(\aobj{}{k+1}, A) \leq (p-k) d(\aobj{}{k}, A) $$
	but we also have 
	$$ \cost{k}{A} \geq k\cdot d(\aobj{}{k}, A) $$
	Then, combine the above two inequalities, we have 
	$$\cost{p}{A} - \cost{k}{A} \leq \frac{p-k}{k} \cost{k}{A} $$
	Recall that $p - k > 0, k > 0$. Thus, we can see that
	$$\frac{k}{p-k}\left[ \cost{p}{A} - \cost{k}{A} \right] \leq \cost{k}{A}.$$
\end{proof} 
\label{lemma:kpReduce}
\end{lemma} 

Now, equipped with the above lemmas, we can form tighter bounds than what we have shown in Theorem \ref{theorem:2_gen}. We begin by looking at the case where $p$ is at least twice as large as $k$. The result follows.\\

\begin{theorem}
	For $1 \leq k < p \leq n$, $\frac{k}{p} \leq \frac{1}{2}$, given the optimal facility location $O_k$ that minimizes $c_k$ and the optimal facility location $O_p$ that minimizes $c_p$, we have that $\alpha_k(O_p) \leq \frac{1}{\alpha_p(O_k)} + 2 - 2 \cdot \frac{k}{p}$.

\begin{proof}
	First, recall that all the clients and facility locations are in metric space $d$, so by triangle inequality we have 
	$$d(O_k, O_p) \leq \frac{1}{p-k}\left[ \sum_{i = k+1}^{p} d(\aobj{p}{i}, O_p) + \sum_{i = k+1}^{p} d(\aobj{p}{i}, O_k) \right]$$ 
	Then, similar to our proof for Theorem \ref{theorem:2_gen}, we have
	\[
	\begin{aligned}
	c_k(O_p) &= \sum_{i=1}^{k} d(\aobj{p}{i}, O_p)\\
	&\leq \sum_{i=1}^{k} \left[d(\aobj{p}{i}, O_k) + d(O_k, O_p)\right]\\
	&= \sum_{i=1}^{k} d(\aobj{p}{i}, O_k) + k\cdot d(O_k, O_p)\\
	&\leq \sum_{i=1}^{k} d(\aobj{p}{i}, O_k) + \frac{k}{p-k}\left[ \sum_{i = k+1}^{p} d(\aobj{p}{i}, O_p) + \sum_{i = k+1}^{p} d(\aobj{p}{i}, O_k) \right]\\
	&= \frac{k}{p-k} \sum_{i = 1}^{p} d(\aobj{p}{i}, O_k) + \left(1 - \frac{k}{p-k}\right) \sum_{i = 1}^{k} d(\aobj{p}{i}, O_k) + \frac{k}{p-k} \sum_{i = k+1}^{p} d(\aobj{p}{i}, O_p)\\
	&= \frac{k}{p-k} \sum_{i = 1}^{p} d(\aobj{p}{i}, O_k) + \left(1 - \frac{k}{p-k}\right) \sum_{i = 1}^{k} d(\aobj{p}{i}, O_k) + \frac{k}{p-k} \left[ \costO{p}{p} - \costO{k}{p}\right]\\
	\end{aligned}
	\]
	Note that for the term $\left(1 - \frac{k}{p-k}\right) \sum_{i = 1}^{k} d(\aobj{p}{i}, O_k)$, since $\frac{k}{p} \leq \frac{1}{2}$, we have $\frac{k}{p-k} \leq 1$, so $1 - \frac{k}{p-k} \geq 0$. Then, we can apply Lemma \ref{lemma:diff_sum} and get $\left(1 - \frac{k}{p-k}\right) \sum_{i = 1}^{k} d(\aobj{p}{i}, O_k) \leq \left(1 - \frac{k}{p-k}\right) \costO{k}{k}$. Then, by Lemma \ref{lemma:diff_sum} and Lemma \ref{lemma:kpsum}, we have
	\[
	\begin{aligned}
	c_k(O_p) &\leq \frac{k}{p-k} \costO{p}{k}  + \left(1 - \frac{k}{p-k}\right) \costO{k}{k} \\
	&\quad + \frac{k}{p-k} \left[ \frac{1}{\alpO{p}{k}} - \alpO{k}{p}\right]\costO{k}{k} + \frac{1}{\alpO{p}{k}} \cdot \frac{k}{p-k}\left[ \costO{p}{k} - \costO{k}{k} \right] \\
	&= \frac{k}{p-k} \left[ \costO{p}{k} - \costO{k}{k} \right] + \costO{k}{k} \\
	&\quad + \frac{k}{p-k} \left[ \frac{1}{\alpO{p}{k}} - \alpO{k}{p}\right]\costO{k}{k} + \frac{1}{\alpO{p}{k}} \cdot \frac{k}{p-k}\left[ \costO{p}{k} - \costO{k}{k} \right] \\
	\end{aligned}
	\]
	Now, by Lemma \ref{lemma:kpReduce}, we can see that
	\[
	\begin{aligned}
	c_k(O_p) &\leq \costO{k}{k} + \costO{k}{k} + \frac{k}{p-k} \left[ \frac{1}{\alpO{p}{k}} - \alpO{k}{p}\right]\costO{k}{k} + \frac{1}{\alpO{p}{k}} \costO{k}{k} \\
	\end{aligned}
	\]
	Divide both side by $\costO{k}{k}$, then
	\[
	\begin{aligned}
	\frac{c_k(O_p)}{\costO{k}{k}} &\leq 1 + 1 + \frac{k}{p-k} \left[ \frac{1}{\alpO{p}{k}} - \alpO{k}{p}\right] + \frac{1}{\alpO{p}{k}} \\
	\left( 1 + \frac{k}{p-k}\right) \alpO{k}{p} &\leq \left( 1 + \frac{k}{p-k}\right)\frac{1}{\alpO{p}{k}} + 2\\
	\alpO{k}{p} &\leq \frac{1}{\alpO{p}{k}} + 2 \cdot \frac{p-k}{p}
	\end{aligned}
	\]
	Hence, we can conclude that 
	$$ \alpha_k(O_p) \leq \frac{1}{\alpha_p(O_k)} + 2 - 2 \cdot \frac{k}{p} $$
\end{proof}
\label{theorem:2_genk}
\end{theorem} 

Note that this result is in a form similar to what we have shown in Theorem \ref{theorem:2_gen} but with an offset of $- 2 \cdot \frac{k}{p}$. What this means is that if we know the value of $\alpO{p}{k}$, then the value of $\alpO{k}{p}$ would be further restricted by this offset comparing to the result in Theorem \ref{theorem:2_gen}. In addition, we can also see that the bigger $\frac{k}{p}$ becomes, the smaller the right-hand side value of the inequality becomes. In other words, assume that $\frac{k}{p} \leq \frac{1}{2}$, as the difference between $k$ and $p$ becomes smaller, the upper bound of $\alpO{k}{p}$ would also becomes smaller given a fixed value of $\alpO{p}{k}$. 

However, we still haven't found the relationship among $\frac{k}{p}, \alpO{k}{p}$ and $\alpO{p}{k}$ when $\frac{1}{2} < \frac{k}{p} \leq 1$. In order to show the underlying relationship between these values, we will utilize a different (much simpler) method from what we have been using for Theorem \ref{theorem:2_gen} and Theorem \ref{theorem:2_genk}, which yields the following results.

\begin{theorem}
For $1 \leq k < p \leq n$, given the optimal facility location $O_k$ that minimizes $c_k$ and the optimal facility location $O_p$ that minimizes $c_p$, we have that $\alpha_k(O_p) \leq \frac{p}{k}\cdot\frac{1}{\alpha_p(O_k)}$ and $\alpha_p(O_k) \leq \frac{p}{k}\cdot\frac{1}{\alpha_k(O_p)}$.

% ---
\begin{proof}	
	Recall that we have
	$$\alpO{p}{k} = \frac{\costO{p}{k}}{\costO{p}{p}} = \frac{\costO{k}{k} + \left[ \costO{p}{k} - \costO{k}{k} \right]}{\costO{k}{p} + \left[ \costO{p}{p} - \costO{k}{p} \right]}$$
	We will first consider the numerator $\costO{p}{k} = \costO{k}{k} + \left[ \costO{p}{k} - \costO{k}{k} \right]$. To simplify notation, define 
	$\psi = \frac{1}{k}\cdot\costO{k}{k},$ so $\costO{k}{k} = k \cdot \psi$.
	Note that here by Lemma \ref{lemma:kpReduce} we must have $\frac{k}{p-k}\cdot \left[ \costO{p}{k} - \costO{k}{k} \right] \leq \costO{k}{k} = k \cdot \psi$, so $\costO{p}{k} \leq k \cdot \psi + (p-k) \cdot \psi = p \cdot \psi$. Then, we will look at the denominator $\costO{k}{p} + \left[ \costO{p}{p} - \costO{k}{p} \right]$. Recall that $\costO{k}{p}=\alpO{k}{p} \cdot \costO{k}{k}$, so we have $\costO{k}{p} = k \cdot \psi \cdot \alpO{k}{p}$. Now, note that by definition we have $\costO{p}{p} - \costO{k}{p} \geq 0$, so we can see that $\costO{p}{p} \geq k \cdot \psi \cdot \alpO{k}{p} + 0 = k \cdot \psi \cdot \alpO{k}{p}$. Therefore, we have
	$$\alpO{p}{k} = \frac{\costO{p}{k}}{\costO{p}{p}} = \frac{\costO{k}{k} + \left[ \costO{p}{k} - \costO{k}{k} \right]}{\costO{k}{p} + \left[ \costO{p}{p} - \costO{k}{p} \right]} \leq \frac{p \cdot \psi}{k \cdot \psi} \cdot \frac{1}{\alpO{k}{p}} = \frac{p}{k} \cdot \frac{1}{\alpO{k}{p}}$$
	Here note that we have $\alpO{k}{p} \geq 1, \alpO{p}{k} \geq 1$ by definition, we can also conclude that
	$$\alpO{k}{p} \leq \frac{p}{k} \cdot \frac{1}{\alpO{p}{k}}$$
	as desired.
	
\end{proof}
\label{theorem:2_genkgen}
\end{theorem}

Interestingly, since the proof for Theorem \ref{theorem:2_genkgen} does not use any properties of metric space, it is true even under non-metric spaces. In addition, note that as the difference between $p$ and $k$ becomes smaller, the value of $\frac{p}{k}$ becomes smaller. This means that the upper bound for $\alpO{k}{p}$ would also become smaller if the value of $\alpO{p}{k}$ is given. And vice versa, the upper bound for $\alpO{p}{k}$ would become smaller if the value of $\alpO{k}{p}$ is given. Now that we have obtained bounds over all $\frac{k}{p} \in (0,1]$, which is equivalent to $\frac{p}{k} \in [1,\infty)$ from Theorem \ref{theorem:2_genk} and Theorem \ref{theorem:2_genkgen}, we can conclude the following results:\\

% ---
\begin{theorem}
For $1 \leq k < p \leq n$, let $x = \frac{p}{k}$. Define $f:[1,\infty)\rightarrow\mathbb{R}$ as
\[f(x) = \begin{cases} 
                \sqrt{{x}} & 1 \leq x \leq 4 \\
                1 - x^{-1} + \sqrt{x^{-2}-{2}{x}^{-1}+2} & x > 4
            \end{cases}
\]
For some fixed $x$, let $\beta=f(x)$. Then there exists a facility location $A$ in $\F$ such that choosing $A$ would give a $\beta$ approximation simultaneously for both minimizing $c_k$ and minimizing $c_p$. In fact, we would either get a $(1, \beta)$ approximation by choosing $O_k$ or a $(\beta, 1)$ approximation by choosing $O_p$. In other words, at least one of $\alpha_k(O_p)$ and $\alpha_p(O_k)$ is less than or equal to $\beta$. Moreover, this result is tight: for each $x$ we give an instance such that all locations in $\F$ are no better than a $\beta$ approximation for at least one of the objectives. 
\label{theorem:2_bound}
\end{theorem}

The above theorem means that $f(\frac{p}{k})$ is a tight upper bound for the approximation ratio we can obtain for two simultaneous objectives in a general metric space.

\begin{proof}

We will first consider the case where $x = \frac{p}{k} > 4$. Then, let $\beta = 1 - x^{-1} + \sqrt{x^{-2}-{2}{x}^{-1}+2}$; we will show that $\max (\min (\alpha_k(O_p), \alpha_p(O_k))) \leq \beta$.  
To do this, similar to the proof for Corollary \ref{coro:2_gen}, we will consider two cases. First, assume $\alpO{p}{k} > \beta$, then by Theorem \ref{theorem:2_genk} we would have 
\[
\begin{aligned}
	\alpha_k(O_p) &\leq \frac{1}{\alpha_p(O_k)} + 2 - 2 \cdot \frac{k}{p}\\
	&< \frac{1}{1 - x^{-1} + \sqrt{x^{-2} - 2x^{-1} +2}} + 2 - 2x^{-1}\\
	&= \frac{1 - x^{-1} - \sqrt{x^{-2} - 2x^{-1} +2}}{\left(1 - x^{-1}\right)^2 - \left[{x^{-2} - 2x^{-1} +2}\right]} + 2 - 2x^{-1}\\
	&= \frac{1 - x^{-1} - \sqrt{x^{-2} - 2x^{-1} +2}}{-1} + 2 - 2x^{-1}\\
	&= -1 + x^{-1} + \sqrt{x^{-2} - 2x^{-1} +2} + 2 - 2x^{-1}\\
	&= 1 - x^{-1} + \sqrt{x^{-2} - 2x^{-1} +2} \\
	&= \beta
	\end{aligned}
\]
Otherwise, we would have $\alpO{p}{k} \leq \beta$. Hence, we can see that $\max (\min (\alpha_k(O_p), \alpha_p(O_k))) \leq \beta$. This result implies that one of $\alpha_k(O_p)$ and $\alpha_p(O_k)$ is less than or equal to $\beta$ as desired. Now, we proceed to show that this bound is tight by constructing an example such that there are only two possible facility locations in $\F$, corresponding to $O_k$ and $O_p$, and $\alpha_p(O_k) = \alpha_k(O_p) = \beta$. Suppose all the clients and facility locations are located on a line (see Figure \ref{fig:lo_bound}). Suppose that there are $k$ clients located at the left-most location, denoted by $A$, with $d(A, O_k) = 1$, $d(A, O_p) = 2 + \delta$, $\delta \geq 0$ such that both $O_k$ and $O_p$ are located to the right of $A$. Additionally, there are $p-k$ clients located to the right of $O_k$, denote the location by $B$,  with $d(B,O_k) = 1$. Lastly, $O_p$ is located to the right of $B$ with $d(B, O_p) = \delta$. Note that in this case we have $\F = \{O_k, O_p\}$, $\C = \{A, B\}$. An illustration of the example is given in Figure \ref{fig:lo_bound}.
	% todo add tikz graph
\begin{figure}[H]
  \begin{center}
            \begin{tikzpicture}[
            roundnode/.style={circle, draw=green!60, fill=green!5, very thick, minimum size=7mm},
            squarednode/.style={rectangle, draw=blue!60, fill=blue!5, very thick, minimum size=5mm},
        ]
            \node[roundnode, label=$k$ clients] (i) {$A$};
            \node[squarednode] (B) [right=2cm of i]{$O_k$};
            \node[roundnode, label=$p-k$ clients] (j) [right=2cm of B]{$B$};
            \node[squarednode] (A) [right=.5cm of j]{$O_p$};
            \draw (i) -- (B) node[draw=none,fill=none,font=\scriptsize,midway,below] {$1$};
            \draw (B) -- (j) node[draw=none,fill=none,font=\scriptsize,midway,below] {$1$};
            \draw (j) -- (A) node[draw=none,fill=none,font=\scriptsize,midway,below] {$\delta$};
        \end{tikzpicture}
  \caption{An instance with two client locations $A, B$ and two possible facility locations $O_k, O_p$ with distances between any two adjacent locations labeled.}
  \label{fig:lo_bound}
  \end{center}
\end{figure}
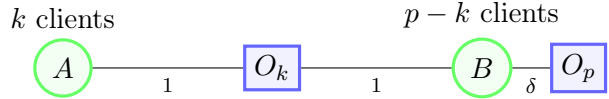

	We claim that by choosing $\delta = \beta - 2$ we would have $\alpO{k}{p} = \alpO{p}{k} = \beta$. To show this, we first need to verify that $\delta \geq 0$. Recall that $x = \frac{p}{k} > 4$, $\beta = f(x) = 1 - x^{-1} + \sqrt{x^{-2}-{2}{x}^{-1}+2}$. For simplicity, let $y = \frac{1}{x}, 0 < y < \frac{1}{4}$, $g(y) = 1 - y + \sqrt{y^2 - 2y +2}$ so $\beta = g(y)$. We will show that $g(y)$ is monotone decreasing in $0 < y < \frac{1}{4}$. Note that 
	\[
	g'(y) = \frac{y-1}{\sqrt{y^2-2y+2}} - 1 = \frac{y-1}{\sqrt{(y-1)^2+1}} - 1 <  \frac{y-1}{y-1} - 1 < 1 - 1 = 0
	\]
	Therefore, we can see that $g(y)$ is monotone decreasing in $0 < y < \frac{1}{4}$, and reaches its minimum value at $\frac{1}{4}$, with $g(\frac{1}{4}) = 1 - \frac{1}{4} + \sqrt{\frac{1}{16} - \frac{1}{2}+2} = 2$, this means that $\beta > 2$, $\delta = \beta - 2 > 2-2 = 0$. Then, for the above example,
	\[
	\begin{aligned}
		\alpO{k}{p} &= \frac{\costO{k}{p}}{\costO{k}{k}} = \frac{(2 + \beta - 2)k}{k} = \beta > 2\\
		\alpO{p}{k} &= \frac{\costO{p}{k}}{\costO{p}{p}} = \frac{k+(p-k)}{(2 + \beta - 2)k + (\beta - 2)(p-k)}\\
		&=  \frac{p}{\beta k + \beta p - 2p - \beta k + 2k}\\
		&=  \frac{p}{\beta p - 2p + 2k} \\
		&= \frac{1}{\beta -2 +2\cdot\frac{k}{p}} \\
		&= \left(1 - y + \sqrt{y^2 - 2y +2} - 2 + 2y\right)^{-1}\\
		&= \left(-1 + y + \sqrt{y^2 - 2y +2}\right)^{-1}\\
		&= 1 - y + \sqrt{y^2 - 2y +2}\\
		&= \beta > 2
	\end{aligned}
	\]
	Since we only have two possible facility locations to choose from, and both $\alpO{k}{p}$ and $\alpO{p}{k}$ are greater than 1, we can verify that both $O_k$ and $O_p$ are indeed the optimal location for their respective objectives. Hence, we can conclude that $\min (\alpha_k(O_p), \alpha_p(O_k)) \geq \beta$, which means that both $\alpha_k(O_p)$ and $\alpha_p(O_k)$ are larger than or equal to $\beta$. Therefore, $\beta$ is a tight upper bound for $\min (\alpha_k(O_p), \alpha_p(O_k))$ when $x = \frac{p}{k} > 4$.
	
	We will now consider the case when $1 \leq x = \frac{p}{k} \leq 4$. Similarly to the previous case, let $\beta = \sqrt{x} = \sqrt{\frac{p}{k}}$, we will show that $\max (\min (\alpha_k(O_p), \alpha_p(O_k))) \leq \beta$. To do this, we will consider two cases. First, assume $\alpO{p}{k} > \sqrt{\frac{p}{k}}$, then by Theorem \ref{theorem:2_genkgen}, we have $\alpha_k(O_p) \leq \frac{p}{k}\cdot \frac{1}{\alpha_p(O_k)} < \frac{p}{k}\cdot \sqrt{\frac{k}{p}} = \sqrt{\frac{p}{k}}$. Otherwise, we would have $\alpO{p}{k} \leq \sqrt{\frac{p}{k}}$. Hence, we can see that $\max (\min (\alpha_k(O_p), \alpha_p(O_k))) \leq \sqrt{\frac{p}{k}} = \beta$. This result implies that one of $\alpha_k(O_p)$ and $\alpha_p(O_k)$ is less than or equal to $\beta$ as desired. Now, to show that bound is tight, we will construct an example such that $\alpha_p(O_k) = \alpha_k(O_p) = \beta$. Suppose in a metric space we have three locations $A$, $O_k$ and $O_p$ such that $\F = \{O_k, O_p\}$, $\C = \{A, O_p\}$. Then, suppose there are $k$ clients located at $A$ and $p-k$ clients located at $O_p$, with $d(O_p, O_k) = 1$, $d(O_k, A) = 1$, $d(O_p, A) = \sqrt{\frac{p}{k}}$. An illustration of the example is given in Figure \ref{fig:up_bound}.
	% todo add tikz graph
	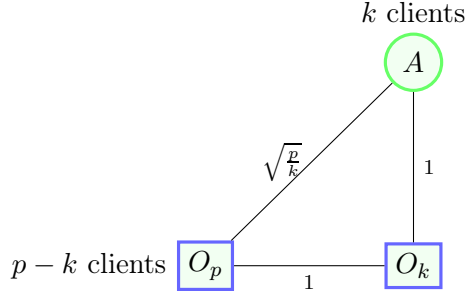
\begin{figure}[H]
  		\begin{center}
  			\begin{tikzpicture}[
            roundnode/.style={circle, draw=green!60, fill=green!5, very thick, minimum size=7mm},
            squarednode/.style={rectangle, draw=blue!60, fill=green!5, very thick, minimum size=5mm},
        ]
            \node[roundnode, label=$k$ clients] (i) {$A$};
            \node[squarednode] (B) [below=2cm of i]{$O_k$};
            \node[squarednode, label=left:$p-k$ clients] (A) [left=2cm of B]{$O_p$};
            \draw (i) -- (B) node[draw=none,fill=none,font=\scriptsize,midway,right] {$1$};
            \draw (B) -- (A) node[draw=none,fill=none,font=\scriptsize,midway,below] {$1$};
            \draw (A) -- (i) node[draw=none,fill=none,font=\scriptsize,midway,left] {$\sqrt{\frac{p}{k}}$};
        \end{tikzpicture}
  			\caption{An instance with two client locations $A, O_p$ and two possible facility locations $O_k, O_p$ with distances between any two adjacent locations labeled.}
  			\label{fig:up_bound}
  		\end{center}
	\end{figure}
	 Recall that $1 \leq \frac{p}{k} \leq 4$, so $1 \leq d(O_p, A) = \sqrt{\frac{p}{k}} \leq 2$. This mean that $d(O_p, O_k) + d(O_k, A) = 2 \geq d(O_p, A)$ so we can verify that the triangle inequality is satisfied, and this example forms a valid metric space.
		In addition, recall that $k < p$, so $\beta = \sqrt{\frac{p}{k}} > 1$. Then, note that for this example
	\[
	\begin{aligned}
		\alpO{k}{p} &= \frac{\costO{k}{p}}{\costO{k}{k}} = \frac{k\cdot d(O_p, A)}{k} = \sqrt{\frac{p}{k}} = \beta > 1\\
		\alpO{p}{k} &= \frac{\costO{p}{k}}{\costO{p}{p}} = \frac{k + p - k}{k\cdot d(O_p, A)}= \frac{p}{k}\cdot \sqrt{\frac{k}{p}} = \sqrt{\frac{p}{k}} = \beta > 1
	\end{aligned}
	\]
	Since we only have two possible facility locations to choose from, and both $\alpO{k}{p}$ and $\alpO{p}{k}$ are greater than 1, we can verify that both $O_k$ and $O_p$ are indeed the optimal location for their respective objectives. Hence, we can conclude that $\min (\alpha_k(O_p), \alpha_p(O_k)) \geq \beta$, which means that both $\alpha_k(O_p)$ and $\alpha_p(O_k)$ are larger than or equal to $\beta$. Then, combined with the results we obtained for the case when $x = \frac{p}{k} > 4$, we can conclude that $\beta$ is a tight upper bound for $\min (\alpha_k(O_p), \alpha_p(O_k))$ for all $x \geq 1$, and that our approximation bounds are tight.
\end{proof}

% --- more work needed 
As a result, as shown in Figure \ref{fig:2_ub}, the function $f(x), x = \frac{p}{k}$ we obtained from Theorem \ref{theorem:2_bound} is continuous and monotone increasing over $\frac{p}{k} \geq 1$. In addition, note that when the difference between $k$ and $p$ is sufficiently large such that the value of $\frac{p}{k}$ approaches $+\infty$, we have $\beta \approx f(\infty) = 1 + \sqrt{2}$, which matches the second result in Corollary \ref{coro:2_gen}. Moreover, note that as the value of $\frac{p}{k}$ approaches 1, the value of $\beta$ also approaches $f(1) = 1$. This shows that the upper bound of the smaller value of $\alpha_k(O_p)$ and $\alpha_p(O_k)$ would approach 1 as the difference between $k$ and $p$ becomes smaller, as expected. 
% In other words, as the difference between $k$ and $p$ becomes smaller, there always exist a facility location such that it will give a better simultaneous approximation ratio for $c_k$ and $c_p$, which eventually would be almost an exact approximation for both objectives, and such ratio is bounded above by $1 + \sqrt{2}$.
%In other words, as the difference between $k$ and $p$ becomes larger, there must exist an outcome such that the approximation ratio for both $c_k$ and $c_p$ induced by it would be within some increasing upper bound ranging from 1 to $1 + \sqrt{2}$ and such bound is tight if we choose between $O_k$ and $O_p$.
In other words, there must always exist an outcome such that the approximation ratio for both $c_k$ and $c_p$ is between 1 and $1 + \sqrt{2}$, this bound is tight, and in fact choosing either $O_k$ or $O_p$ is enough to achieve it.

% \begin{figure}[H]
%   \begin{center}
%   \includegraphics[width=3in]{2_ub_rev.pdf}
%   \caption{Plot of the upper bound of the simultaneous approximation ratio for $c_k$ and $c_p$ (represented as function $f$) with respect to $\frac{p}{k}$ over $[1,20]$.}
%   \label{fig:2_ub}
%   \end{center}
% \end{figure}

\section{Simultaneously Approximating Multiple Objectives}
\label{section:mul_obj}
Now that we have found a tight upper bound for the approximation ratio for two objectives, we want to see what would happen if we have more objectives. 
Assume we have a set of $q \geq 2$ distinct integers in $[1, n]$, $|\C| = n$, $\K = \{k_1, k_2, \cdots, k_q\}$, arranged in increasing order such that $k_1 < k_2 < \cdots < k_q$. Then, the set of objectives that we would like to simultaneously optimize is $\K$ \footnote{As defined in Section \ref{section:prelim}, the set of objectives considered is $\{c_{k_1}, c_{k_2}, \cdots, c_{k_q}\}$, but is denoted by the set of integers $\K$.}.
% $\{c_{k_1}, c_{k_2}, \cdots, c_{k_q}\}$. 
First, immediately from Corollary \ref{coro:2_gen}, we can get the following result:

\begin{theorem}
Consider the optimal facility location $O_{n}$ that minimizes $c_{n}$. By picking $O_{n}$, we obtain a 3 approximation for all other objectives $c_k$ for $k\leq n$.
%For $\sS = \{c_{k_1}, c_{k_2}, \cdots, c_{k_q}\}$ with $1 \leq k_1 < k_2 < \cdots < k_q \leq n$, consider the optimal facility location $O_{k_q}$ that minimizes $c_{k_q}$. By picking $O_{k_q}$, we obtain a 3 approximation for all other objectives in $\sS$.

\begin{proof}
%Note that for any $k \in \K \setminus \{k_q\}$, we have that $k_q > k$, so by Corollary \ref{coro:2_gen}, we have $\alpO{k}{k_q} \leq 3$. Therefore, we can conclude that by picking $O_{k_q}$, we would obtain a 3 approximation for every other objectives.
Note that for any $k < n$, by Corollary \ref{coro:2_gen}, we have $\alpO{k}{n} \leq 3$. Therefore, we can conclude that by picking $O_{n}$, we would obtain a 3 approximation for every other objective.
\end{proof}
\label{theorem:mul3}
\end{theorem}

While the above theorem gives us a very simple way of obtaining a 3-approximation for all objectives (and in fact can also be obtained as a simple consequence of the results in \cite{goel2018relating}), we are interested in a more fine-grained analysis of when better approximations are possible. What if we are only interested in approximating a few objectives simultaneously, instead of all of them (in other words, what if $|\K|$ is small)? Or what if the set $\K$ has some nice properties? Towards answering these questions, we first make the following observation from Theorem \ref{theorem:2_genkgen}:

\begin{corollary}
For $1 \leq k < p \leq n$, we have $\alpha_k(O_p) \leq \frac{p}{k}$ and $\alpha_p(O_k) \leq \frac{p}{k}$.

\begin{proof}
Assume $1 \leq k < p \leq n$. Recall that Theorem \ref{theorem:2_genkgen} states that $\alpha_k(O_p) \leq \frac{p}{k}\cdot\frac{1}{\alpha_p(O_k)}$ and $\alpha_p(O_k) \leq \frac{p}{k}\cdot\frac{1}{\alpha_k(O_p)}$. But by definition we have $\alpha_p(O_k) \geq 1$ and $\alpha_p(O_k) \geq 1$, so we can conclude that $\alpha_k(O_p) \leq \frac{p}{k}$ and $\alpha_p(O_k) \leq \frac{p}{k}$.
\end{proof}
\label{coro:gen_re}
\end{corollary}

Corollary \ref{coro:gen_re} indicates that when the difference between any two objectives $c_k$ and $c_p$ is sufficiently small, both $\alpO{p}{k}$ and $\alpO{k}{p}$ would also be small. One direct observation we can see from this is when $\K = \{k, k+1, \cdots, 4k-1, 4k\}$ for some $1 \leq k \leq n$, by picking $O_{2k}$ we can get a 2 approximation for every other objective in $\K$. The reason for this is because $2k$ is of a factor of 2 larger than the smallest element in $\K$ and of a factor of 2 smaller than largest element in $\K$. Then by Corollary \ref{coro:gen_re} we must have for any $k \in \K \setminus \{2k\}$, $\alpO{k}{2k} \leq 2$. Note that this is a better result than what we have shown in Theorem \ref{theorem:mul3} but is only true for special cases when none of the objectives in $\K$ are very different. Now, we want to see if we can obtain a better result for multiple objectives in general. To do this, we will first construct a graph representation for this problem.

We construct a complete directed graph $G = (V,E)$ as follows. First, for each $k \in \K$, we will make a node representing $O_k$, which is the optimal facility location for objective $c_k$. For simplicity, we will denote this node by $O_k$. Then, for every pair $i,j \in \K$, $i \neq j$, we will make two edges $(O_i, O_j)$ and $(O_j, O_i)$ with weight $\alpO{j}{i}$ and $\alpO{i}{j}$ respectively. As an example, Figure \ref{fig:graph_rep} is an illustration of $G$ for three objectives $c_{i}$, $c_{j}$ and $c_k$. Note that for every node $O_k$, $k \in \K$, by choosing $O_k$, we would get a $(\alpha_{k_1}(O_k), \alpha_{k_2}(O_k), \cdots, \alpha_{k_q}(O_k))$ approximation for minimizing the set of objectives $\K$ but these values are exactly the weights of all the edges going out of node $O_k$. Therefore, instead of looking at individual approximation ratios and their relationship, we will utilize this graph representation. 

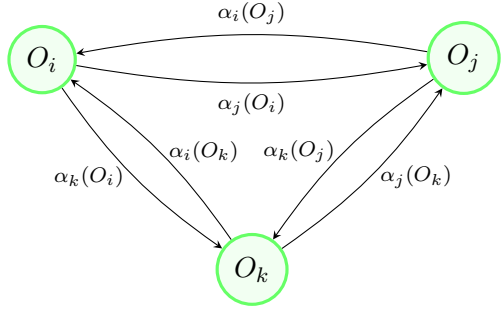
\begin{figure}[H]
  \begin{center}
    \begin{tikzpicture}
        [roundnode/.style={circle, draw=green!60, fill=green!5, very thick, minimum size=7mm},
            squarednode/.style={rectangle, draw=red!60, fill=red!5, very thick, minimum size=5mm},
        ]
            \node[roundnode] (k) {$O_k$};
            \node[roundnode] (i) [above left = 3cm of k]{$O_i$};
            \node[roundnode] (j) [above right = 3cm of k]{$O_j$};
            \draw[-stealth] (k) to [bend right=10] node[draw=none,fill=none,font=\scriptsize,midway,right] {$\alpha_i(O_k)$} (i);
            \draw[-stealth] (i) to [bend right=10] node[draw=none,fill=none,font=\scriptsize,midway,left] {$\alpha_k(O_i)$} (k);
            \draw[-stealth] (i) to [bend right=10] node[draw=none,fill=none,font=\scriptsize,midway,below] {$\alpha_j(O_i)$} (j);
            \draw[-stealth] (j) to [bend right=10] node[draw=none,fill=none,font=\scriptsize,midway,above] {$\alpha_i(O_j)$} (i);
            \draw[-stealth] (j) to [bend right=10] node[draw=none,fill=none,font=\scriptsize,midway,left] {$\alpha_k(O_j)$} (k);
            \draw[-stealth] (k) to [bend right=10] node[draw=none,fill=none,font=\scriptsize,midway,right] {$\alpha_j(O_k)$} (j);
    \end{tikzpicture}
  \caption{An example graph representation of three objectives, $c_{i}$, $c_{j}$ and $c_k$. The green circles are nodes, arrows are directed edges with their respective weight labeled.}
  \label{fig:graph_rep}
  \end{center}
\end{figure}

Our goal is to find some value $\beta_q < 3$ such that for $|\K| = q$, there must exist some $k \in \K$ such that choosing $O_k$ would be at worst a $\beta_q$ approximation for every other objective in $\K$. Note that the objectives in $\K$ can be arbitrarily far apart, for example they could include the max $c_1$ and the sum $c_n$ objectives. In fact, we can obtain the following results.\\
%To show this, we will first make several observations: \\

\begin{theorem}
For the directed graph $G$ defined above, let $|\K| = |V| = q$, and consider $\beta_j$ that satisfies
$$(\beta_j - 2)^{j-1}\beta_j = 1, \quad \beta_j \geq 1 + \sqrt{2}$$
for some $2 \leq j \leq q$. Then there does not exist a cycle of size j such that all of the edges in the cycle have weight strictly larger than $\beta_j$.
\label{theorem:mul_cyc}
\end{theorem}

\begin{theorem}
Let $|\K| = q$, and consider $\beta_q$ as defined in Theorem \ref{theorem:mul_cyc}. Then there must exist some $k \in \K$ such that choosing $O_k$ would be at worst a $\beta_q$ approximation for every other objective in $\K$.
\label{theorem:mul_gen}
\end{theorem}

Note that Theorem \ref{theorem:mul_cyc} reveals the properties of the cycles in the graph $G$, while Theorem \ref{theorem:mul_gen} indicates that $\beta_q$ defined in Theorem \ref{theorem:mul_cyc} is the upper bound of the approximation ratio for simultaneously approximating all the objectives in $\K$, which is exactly what we are looking for.

To show the above theorems, we will first make several observations. First, note that we introduced a new value $\beta_j$, for any $j \in [2,q]$. In fact, there are some useful properties of $\beta_j$. These properties are:\\

\begin{lemma} 
Consider $\beta_j$ as defined in Theorem \ref{theorem:mul_cyc}, with $j \geq 2, j \in \mathbb{N}$, then
\begin{enumerate}
    \item $\beta_j$ exists and is unique.
    \item $1 + \sqrt{2} \leq \beta_j < 3$
	\item $\beta_j$ increases as the value of j increases. 
	
\end{enumerate}

\begin{proof} Let $f(x) = (x-2)^{j-1}x - 1$ for some $j \geq 2$, then $\beta_j$ is the root of $f(x)$ that is larger than or equal to $1 + \sqrt{2}$ by definition. First, suppose that such $\beta_j$ exists; we will show that $1 + \sqrt{2} \leq \beta_j < 3$. To do this, suppose there exists some $\beta_j$ such that $\beta_j \geq 3$. Recall that $j \geq 2, j-1 \geq 1$, and $\beta_j \geq 3, \beta_j - 2 \geq 1$, so $(\beta_j-2)^{j-1} \geq 1$. Then, we have  
$$f(\beta_j) = (\beta_j-2)^{j-1}\beta_j - 1 \geq 1 \cdot 3 - 1 = 2$$
But $f(\beta_j) = 0$, this is a contradiction. Therefore, we must have $\beta_j < 3$. 

Now we will show that $f(x)$ has a unique root within the range $[1+\sqrt{2}, 3)$. Let $1 + \sqrt{2} \leq x < 3$, so $x - 2 > 0$. Note that we also have $j-1 \geq 1$, so $g(x) = (x-2)^{j-1}$ is continuous on $x \in [1+\sqrt{2}, 3)$. In addition, since $h(x) = x$ is also continuous on $x \in [1+\sqrt{2}, 3)$, we have that $f(x) = g(x)\cdot h(x) - 1$ is continuous on $x \in [1+\sqrt{2}, 3)$. Now, note that we have 
$$f'(x) = (x-2)^{j-2}(jx-2).$$
Here recall that we have $x-2 > 0$, so $(x-2)^{j-2} > 0$. Besides, we have $j \geq 2$ so $jx - 2 \geq 2(1 + \sqrt{2})-2 > 0$. This means that
$$f'(x) > 0\cdot 0 = 0$$
Therefore, $f(x)$ is monotone increasing on $x \in [1+\sqrt{2}, 3)$.
Moreover, we have that 
$$f(1+\sqrt{2}) = (\sqrt{2}-1)^{j-1}(1+\sqrt{2})-1 = (\sqrt{2}-1)^{j-2}-1$$
Here note that $\sqrt{2}-1 < 1$, $j-2 \geq 0$, so $(\sqrt{2}-1)^{j-2} \leq 1$, then we have 
$$f(1+\sqrt{2}) \leq 1-1 = 0$$
Likewise, we can see that
$$f(3) = (3-2)^{j-1}(3)-1 = (1)^{j-1}\cdot 3-1 = 3-1 = 2$$
Then, since we have shown that $f(x)$ is continuous on $x \in [1+\sqrt{2}, 3)$, $f(1+\sqrt{2}) \leq 0$, $f(3) = 2$, by the intermediate value theorem there must exist some $\beta_j \in [1+\sqrt{2}, 3)$ such that $f(\beta_j)=0$. In addition, since $f(x)$ is monotone increasing on $x \in [1+\sqrt{2}, 3)$, such $\beta_j$ is unique. Therefore, we can conclude that $\beta_j$ exists and is unique for any $j \geq 2$.

Lastly, we will show that $\beta_j$ increases as the value of $j$ increases. Suppose otherwise, there exist some $c \geq 2, d \geq 2$ such that $\beta_c < \beta_d$ but $c > d$. By definition we know that $(\beta_c-2)^{c-1}\beta_c = 1, (\beta_d-2)^{d-1}\beta_d = 1$. Now since $c > d$, $c - d \geq 1$, we have 
$$(\beta_c-2)^{c-1}\beta_c = (\beta_c-2)^{c-d}(\beta_c-2)^{d-1}\beta_c < (\beta_d-2)^{c-d}(\beta_d-2)^{d-1}\beta_d = (\beta_d-2)^{c-d}$$
Here note that we have shown $\beta_d < 3$, so $\beta_d-2 < 3-2 = 1$, we then have 
$$(\beta_c-2)^{c-1}\beta_c < 1^{c-d} = 1$$
But $(\beta_c-2)^{c-1}\beta_c = 1$, this is a contradiction. Hence, we must have $c < d$. This means that $\beta_j$ increases as the value of $j$ increases. 
% \begin{figure}[H]
%   \begin{center}
%   \includegraphics[width=2.5in]{j_ub.pdf}
%   \caption{Plot of $\beta_j$ over $j \in [2,20]$.}
%   \label{fig:mul_ub}
%   \end{center}
% \end{figure}
\end{proof}
\label{lemma:beta}
\end{lemma}

Then, we continue with some useful lemmas.\\

\begin{lemma}
For $1 \leq k < p \leq n$, we have $\costO{p}{p} \geq \frac{p}{k} \left[\alpO{k}{p}-2\right] \cdot \costO{k}{k}$.

\begin{proof}
Recall that all the facility locations are in a metric space, and thus distances obey the triangle inequality. Similar to the proof for Theorem \ref{theorem:2_gen}, we have 
\[
\begin{aligned}
	\costO{k}{p} &\leq \sum_{i = 1}^{k} d(\aobj{p}{i}, O_k) + k \cdot d(O_k, O_p)\\
	&\leq \sum_{i = 1}^{k} d(\aobj{p}{i}, O_k) + \frac{k}{p} \left[\costO{p}{p} +  \sum_{i = 1}^{p} d(\aobj{p}{i}, O_k)\right]
\end{aligned}
\]
Then, by Lemma \ref{lemma:diff_sum}, we have 
\[
\begin{aligned}
	\costO{k}{p} \leq \costO{k}{k} + \frac{k}{p} \left[\costO{p}{p} +  \costO{p}{k}\right]
\end{aligned}
\]
Now, note that $\frac{p}{k} > 0$, rearrange the above inequality, we can see that
$$\costO{p}{p} \geq \frac{p}{k}\cdot \costO{k}{p} - \frac{p}{k}\cdot \costO{k}{k} - \costO{p}{k}$$
Recall that $\costO{k}{p} = \alpO{k}{p} \cdot \costO{k}{k}$ by definition and since $k < p$, by Lemma \ref{lemma:k_cost} we have $\costO{p}{k} \leq \frac{p}{k}\cdot \costO{k}{k}$, then
\[
\begin{aligned}
	\costO{p}{p} &\geq \frac{p}{k}\cdot \alpO{k}{p} \cdot \costO{k}{k} - \frac{p}{k}\cdot \costO{k}{k} - \frac{p}{k}\cdot \costO{k}{k}\\
	&= \frac{p}{k}\cdot \left[\alpO{k}{p}-2\right] \cdot \costO{k}{k}
\end{aligned}
\]
as desired.
\end{proof}
\label{lemma:ptok}
\end{lemma}

\begin{lemma}
For $1 \leq k < p \leq n$, we have $\costO{k}{k} \geq \frac{k}{p} \cdot \alpO{p}{k} \cdot \costO{p}{p}$.

\begin{proof}
Note that since $k < p$, by Lemma \ref{lemma:k_cost} we have $\costO{p}{k} \leq \frac{p}{k}\cdot \costO{k}{k}$ and we have $\costO{p}{k} = \alpO{p}{k} \cdot \costO{p}{p}$ by definition and $\frac{k}{p} > 0$, then we have
$$\costO{k}{k} \geq \frac{k}{p}\cdot \costO{p}{k} = \frac{k}{p}\cdot \alpO{p}{k} \cdot \costO{p}{p}$$
\end{proof}
\label{lemma:ktop}
\end{lemma}

With the above lemmas, the proof for Theorem \ref{theorem:mul_cyc} follows.\\

\begin{proof}[Proof of Theorem \ref{theorem:mul_cyc}]
To show this, suppose that there is a cycle of size $j$, $2 \leq j \leq q$, $|\K| = |V| = q$ such that all of its edge weights are larger than or equal to some $\beta > 1$. For simplicity, denote the edges involved in this cycle by $(O_{k_1}, O_{k_2}), (O_{k_2}, O_{k_3}), \cdots, (O_{k_{j-1}}, O_{k_j}),(O_{k_j}, O_{k_1})$, where $\{k_1, k_2, \cdots, k_j\} \subseteq \K$. So we have $\alpO{k_2}{k_1} \geq \beta, \alpO{k_3}{k_2} \geq \beta, \cdots, \alpO{k_{j}}{k_{j-1}} \geq \beta, \alpO{k_1}{k_j} \geq \beta$. Now, for some $k,p$, by Lemma \ref{lemma:ptok} and Lemma \ref{lemma:ktop}, we have 
\begin{align}
    \costO{k}{k} &\geq \gamma_{p}^{k} \cdot \frac{k}{p} \cdot \costO{p}{p}\\
	\gamma_{p}^{k} &= \begin{cases} 
                \alpO{p}{k} & k < p \\
                \alpO{p}{k}-2 & k > p 
            \end{cases}
\label{ineq1}
\end{align}
Here note that $\alpO{p}{k} > \alpO{p}{k}-2$, so we must have $\gamma_{p}^{k} \geq \alpO{p}{k}-2$. Then, by (1) we have that 
\[
\begin{aligned}
	\costO{k_1}{k_1} &\geq \frac{k_1}{k_2}\cdot \gamma_{k_2}^{k_1} \cdot\costO{k_2}{k_2}\\
	&\geq \frac{k_1}{k_2} \cdot \frac{k_2}{k_3}\cdot \gamma_{k_2}^{k_1} \cdot \gamma_{k_3}^{k_2} \cdot\costO{k_3}{k_3}\\
	&\geq \frac{k_1}{k_2} \cdot \frac{k_2}{k_3}\cdot \cdots \cdot \frac{k_{j-1}}{k_j} \cdot \frac{k_j}{k_1} \cdot \gamma_{k_2}^{k_1} \cdot \gamma_{k_3}^{k_2} \cdot \cdots \cdot \gamma^{k_{j-1}}_{k_j} \cdot \gamma^{k_j}_{k_1} \cdot\costO{k_1}{k_1}\\
	&=  \gamma_{k_2}^{k_1} \cdot \gamma_{k_3}^{k_2} \cdot \cdots \cdot \gamma^{k_{j-1}}_{k_j} \cdot \gamma^{k_j}_{k_1} \cdot\costO{k_1}{k_1}
\end{aligned}
\]
Recall that $\gamma_{p}^{k} \geq \alpO{p}{k}-2$ for any $p,k$, but note that for at least one of the $\gamma_{k_y}^{k_x}$ values above, it must that that $k_x < k_y$, since otherwise we would get that $k_1 > k_1$, a contradiction.
%must exist one $k_x, k_y$ in $\{k_1, k_2, \cdots, k_j\}$ such that $k_x < k_y$ so $\gamma_{k_y}^{k_x} = \alpO{k_y}{k_x}$. Assume otherwise, then we must have $k_1 > k_2, k_2 > k_3, \cdots, k_{j-1} > k_j, k_j > k_1$. But we then have $k_1 > k_j$ and $k_j > k_1$ which is a contradiction. 
Therefore, there must exist at least one $\gamma_{k_y}^{k_x} = \alpO{k_y}{k_x}$. Without loss of generality, let $\gamma_{k_2}^{k_1} = \alpO{k_2}{k_1}$, then we can see that
\[
\begin{aligned}
	\costO{k_1}{k_1} &\geq \alpO{k_2}{k_1} \cdot \left[\alpO{k_3}{k_2}-2\right] \cdot \cdots \cdot  \left[\alpO{k_j}{k_{j-1}}-2\right] \cdot  \left[\alpO{k_1}{k_j}-2\right] \cdot\costO{k_1}{k_1}\\
	&\geq (\beta - 2)^{j-1}\beta \cdot\costO{k_1}{k_1}
\end{aligned}
\]
Here we know that $\costO{k_1}{k_1} > 0$ by definition, so
$$(\beta - 2)^{j-1}\beta \leq 1$$
% check this
Note that we must have $\beta > 1$.
Now, consider $\beta_j$ such that $(\beta_j - 2)^{j-1}\beta_j = 1$, $\beta_j \geq 1+\sqrt{2} > 1$. By Lemma \ref{lemma:beta}, we know that $\beta_j$ must exist and is unique. We then claim that $\beta \leq \beta_j$. To show this, suppose $\beta > \beta_j$. Similar to the proof of Lemma \ref{lemma:beta}, let $f(x) =  (x - 2)^{j-1}x$, then we have $f'(x) = (x-2)^{j-2}(j x-2)$. Consider $x \geq \beta_j$, then we have $x - 2 \geq \beta_j - 2 \geq 1+\sqrt{2} - 2 > 0$ and $j x - 2 \geq 2  (1+\sqrt{2}) - 2 >0$, so $f'(x) > 0 \cdot 0 = 0 $. This means that $f(x)$ is monotone increasing when $x \geq \beta_j$, so $(\beta - 2)^{j-1}\beta > (\beta_j - 2)^{j-1}\beta_j = 1$. But $(\beta - 2)^{j-1}\beta \leq 1$, this is a contradiction. Hence, we must have $\beta \leq \beta_j$.
% and the largest value $\beta$ can be is some $\beta_j$ such that $(\beta_j - 2)^{j-1}\beta_j = 1$, $\beta_j > 1$. Then, by Lemma \ref{lemma:beta} there must exist a unique $\beta_j \geq 1 + \sqrt{2} > 1$ for $j \geq 2$. Here note that the uniqueness of $\beta_j$ ensures that   such $\beta_j$ is indeed the upper bound of $\beta$. 
Recall that we have $\alpO{k_2}{k_1} \geq \beta, \alpO{k_3}{k_2} \geq \beta, \cdots, \alpO{k_{j}}{k_{j-1}} \geq \beta, \alpO{k_1}{k_j} \geq \beta$, but $\beta \leq \beta_j$, this means that we cannot have $\alpO{k_2}{k_1} > \beta_j, \alpO{k_3}{k_2} > \beta_j, \cdots, \alpO{k_{j}}{k_{j-1}} > \beta_j, \alpO{k_1}{k_j} > \beta_j$. In other words, one of the $\alpha's$ has to be smaller than or equal to $\beta_j$. Since the choice of the cycle of size j is arbitrary, we can conclude that there does not exist a cycle of size j such that all of the edges in the cycle have weight strictly larger than $\beta_j$.
\end{proof}

Then, finally, we present the proof for Theorem \ref{theorem:mul_gen}:\\

\begin{proof}[Proof of Theorem \ref{theorem:mul_gen}] 
Recall that in the graph representation $G$, for some node $O_k$, $k \in \K$, every edge that goes out of $O_k$ into some $O_p$ has weight representing the approximation ratio using $O_k$ with respect to objective $c_{p}$, denoted by $\alpO{p}{k}$. Since the graph $G$ is complete, we will show the above theorem by showing that there must exist some $O_k$, $k \in \K$, $|\K| = q$ such that all of the edges leaving $O_k$ have weight less than or equal to $\beta_q$. We will show this by contradiction. Suppose otherwise, for every $k \in \K$, we must have at least one of the edges going out of $O_k$ having weight larger than $\beta_q$. Denote this set of edges by $E'$, note that $|E'| = q$. We will then consider these edges. Here note that since each node would not have an edge going into itself and every edge in $E'$ has a distinct starting node, a subset of $E'$ must be able to form a cycle of size $j \leq q$, denote the cycle by $C_j$. By Theorem \ref{theorem:mul_cyc} we know that $G$ cannot have a cycle of size $j$ such that all of it edge weights are larger than $\beta_j$. However, by Lemma \ref{lemma:beta}, since we have $j \leq q$, we must also have $\beta_q \geq \beta_j$ so all the edges in $C_j$ have weight larger than $\beta_j$, which is a contradiction. Therefore, we can conclude that there must exist some $O_k$, $k \in \K$, such that all of the edges leaving $O_k$ have weight less than or equal to $\beta_q$. This means that choosing $O_k$ would be at worst a $\beta_q$ approximation for every other objective in $\K$.
\end{proof}

Note that unlike Theorem \ref{theorem:mul3}, Theorem \ref{theorem:mul_gen} shows that an outcome can always be chosen with approximation ratio better than 3 for every other objective (see Figure \ref{fig:mul_ub}). This is true since we have shown that $\beta_q < 3$ in Lemma \ref{lemma:beta}. %In addition, as we can see in Figure \ref{fig:mul_ub}, the upper bound of the approximation ratio $\beta_q$ gets worse as the number of objectives increases. 
In fact, note that $\beta_2 = 1 + \sqrt{2}$, which matches with the tight bound we obtained in Theorem \ref{theorem:2_bound} for approximating two simultaneous objectives in the case where the difference between the two objectives is allowed to be arbitrary. As the number of objectives we care about grows, so does the approximation factor, and it can become strictly larger than $1+\sqrt{2}$, as we show in the following proposition:

\begin{proposition}
	There exists an instance with 3 objectives such that all possible facility locations in $\F$ result in at least $\beta_3$ approximation for at least one of the three objectives.
	
\begin{proof}
	We will show this by constructing an example. Assume $\K = \{1, k, n\}$, with $1 < k < n$, $|\C| = n$. Consider three locations $A$, $B$ and $C$, where there is one client located at $A$, $k-1$ clients located at $B$ and $n-k$ clients located at $C$. Note that $\beta_3 = \frac{1}{2}(3 + \sqrt{5})$, then consider three more locations $O_1$, $O_k$ and $O_n$ with distances to $A$, $B$ and $C$ as shown in Table \ref{tab:3bound1} below:
	
\begin{table}[H]
\centering
\begin{tabular}{|c|ccc|}
\hline
 \cline{2-4} 
\multicolumn{1}{|l|}{$O$}                  & \multicolumn{1}{c|}{$d(O, A)$}      & \multicolumn{1}{c|}{$d(O, B)$}              & $d(O, C)$ \\ \hline
$O_n$                                   & \multicolumn{1}{c|}{$\frac{1}{2}(1 + \sqrt{5})$} & \multicolumn{1}{c|}{$\frac{1}{2}(1 + \sqrt{5})$} & $\frac{1}{2}(3 - \sqrt{5})$  \\ \hline
$O_k$                                   & \multicolumn{1}{c|}{$\frac{1}{2}(3 + \sqrt{5})$}           & \multicolumn{1}{c|}{$\frac{1}{2}(\sqrt{5} - 1)$}                   & $\frac{1}{2}(\sqrt{5} - 1)$      \\ \hline
$O_1$                                   & \multicolumn{1}{c|}{1}                   & \multicolumn{1}{c|}{1}                             & 1                \\ \hline
\end{tabular}
\caption{An instance for Proposition \ref{prop:3ub}, with distances between $O_1, O_k, O_n$ and $A, B, C$ listed.}
\label{tab:3bound1}
\end{table}
\noindent
Here note that we have $d(O_n, A) = d(O_n, B) > d(O_n, C)$, $d(O_k, A) > d(O_k, B) = d(O_k, C)$ and $d(O_1, A) = d(O_1, B) = d(O_1, C)$.
Assume that $k$ and $n$ are sufficiently large such that $k-1 \approx k, n - k \approx n$ so we have the (approximate) costs as shown in Table \ref{tab:3bound2} below:
\begin{table}[H]
\centering
\begin{tabular}{|c|ccc|}
\hline
 \cline{2-4} 
\multicolumn{1}{|l|}{$O$}                  & \multicolumn{1}{c|}{$\costO{1}{}$}      & \multicolumn{1}{c|}{$\costO{k}{}$}              & $\costO{n}{}$ \\ \hline
$O_n$                                   & \multicolumn{1}{c|}{$\frac{1}{2}(1 + \sqrt{5})$} & \multicolumn{1}{c|}{$\frac{k}{2}(1 + \sqrt{5})$} & $\frac{n}{2}(3 - \sqrt{5})$  \\ \hline
$O_k$                                   & \multicolumn{1}{c|}{$\frac{1}{2}(3 + \sqrt{5})$}           & \multicolumn{1}{c|}{$\frac{k}{2}(\sqrt{5} - 1)$}                   & $\frac{n}{2}(\sqrt{5} - 1)$      \\ \hline
$O_1$                                   & \multicolumn{1}{c|}{1}                   & \multicolumn{1}{c|}{$k$}                             & $n$                \\ \hline
\end{tabular}
\caption{An instance for Proposition \ref{prop:3ub}, with $c_1$, $c_k$ and $c_n$ listed for $O_1, O_k$ and $O_n$.}
\label{tab:3bound2}
\end{table}
\noindent
Note that in this case we have $\C = \{A, B, C\}$ and $\F = \{O_1, O_k, O_n\}$. To construct an example that matches Table \ref{tab:3bound1}, we will consider an instance as shown in Figure \ref{fig:3_ub} below:

\begin{figure}[H]
  \begin{center}
  \begin{tikzpicture}[
            roundnode/.style={circle, draw=green!60, fill=green!5, very thick, minimum size=7mm},
            squarednode/.style={rectangle, draw=blue!60, fill=blue!5, very thick, minimum size=5mm},
            scale = 1.5, transform shape
        ]

            \node[squarednode] (k) {$O_k$};
            \node[roundnode] (B) [right = 2cm of k]{$B$};
            \node[roundnode] (A) [right = 2cm of B]{$A$};
            \node[roundnode] (C) [below = 2cm of k]{$C$};
            \node[squarednode] (n) [right = 2cm of C]{$O_n$};
            \node[squarednode] (i) [right = 2cm  of n]{$O_1$};
            
            \draw (k) -- (C) node[draw=none,fill=none,font=\scriptsize,midway,left ] {$\frac{1}{2}\left(\sqrt{5}-1\right)$};
            \draw (k) -- (B) node[draw=none,fill=none,font=\scriptsize,midway, above] {$\frac{1}{2}\left(\sqrt{5}-1\right)$};
            \draw (C) -- (n) node[draw=none,fill=none,font=\scriptsize,midway, below] {$\frac{1}{2}\left(3 - \sqrt{5}\right)$};
            \draw (n) -- (i) node[draw=none,fill=none,font=\scriptsize,midway,below] {$\frac{1}{2}\left(\sqrt{5}-1\right)$};
            \draw (A) -- (i) node[draw=none,fill=none,font=\scriptsize,midway,right] {$1$};
            \draw (i) -- (B) node[draw=none,fill=none,font=\scriptsize,midway, right] {$1$};
            
            \draw (B) -- (n) node[draw=none,fill=none,font=\scriptsize,midway,below right] {$\frac{1}{2}\left(1+\sqrt{5}\right)$};
            \draw (B) -- (C) node[draw=none,fill=none,font=\scriptsize,midway,left] {$\sqrt{5}-1$};
            \draw (B) -- (A) node[draw=none,fill=none,font=\scriptsize,midway,above left] {2};
        \end{tikzpicture}
  \caption{An instance with three client locations $A, B, C$ and three possible facility locations $O_1, O_k, O_n$ with distances between any two adjacent locations labeled.}
  \label{fig:3_ub}
  \end{center}
\end{figure}
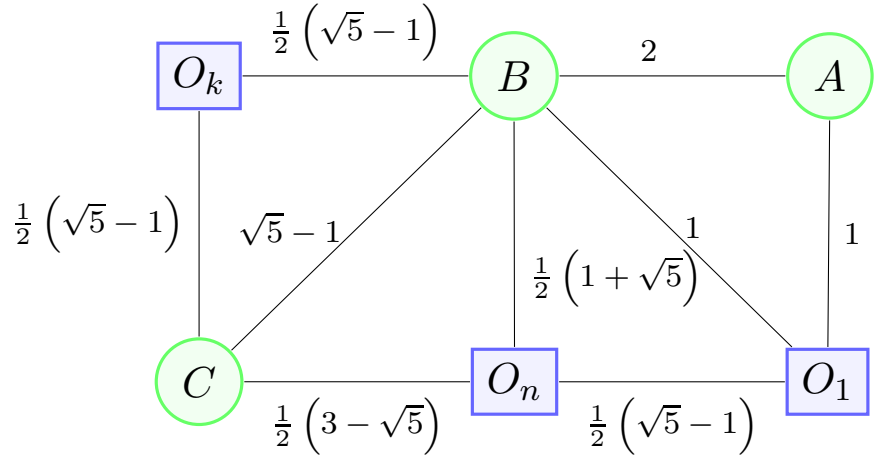
\noindent
Note that we have $d(O_k, C) = \frac{1}{2}(\sqrt{5}-1), d(O_n, C) = \frac{1}{2}(3-\sqrt{5}), d(O_n, O_1) = \frac{1}{2}(\sqrt{5}-1), d(O_1, k) = 1$ and $d(O_k, B) = \frac{1}{2}(\sqrt{5}-1), d(C, B) = \sqrt{5}-1, d(O_n, B) = \frac{1}{2}(1+\sqrt{5}), d(O_1, B) = 1, d(A, B) = 2$, we will consider the graph metric implied in Figure \ref{fig:3_ub} such that $d(O_k, O_n) = d(O_k, C) + d(C, O_n) = 1, d(O_k, O_1) = d(O_k, O_n) + d(O_n, O_1) = \frac{1}{2}(1+\sqrt{5}), d(O_k, A) = d(O_k, O_1) + d(O_1, A) = \frac{1}{2}(3+\sqrt{5}), d(C, O_1) = d(C,O_n) + d(O_n, O_1) = 1, d(C, A) = d(C,O_1) + d(O_1, A) = 2, d(O_n, A) = d(O_n, O_1) + d(O_1, A) = \frac{1}{2}(1+\sqrt{5})$. Note that these values match the values in Table \ref{tab:3bound1}. In addition, these distances satisfy triangle inequality thus the example is valid under graph metric.

Then, using the values listed in Table \ref{tab:3bound2}, note that $\costO{1}{1} < \costO{1}{n} < \costO{1}{k}, \costO{k}{k} < \costO{k}{1} < \costO{k}{n}, \costO{n}{n} < \costO{n}{k} < \costO{n}{1}$, so $O_1$ is the optimal location for $c_1$, $O_k$ is the optimal location for $c_k$ and $O_n$ is the optimal location for $c_n$. Hence we can see that
\[
\begin{aligned}
	\alpO{1}{k} &= \frac{\costO{1}{k}}{\costO{1}{1}} = \frac{\frac{1}{2}(3+\sqrt{5})}{1} =  \frac{1}{2}(3+\sqrt{5}) = \beta_3 > \frac{\costO{1}{n}}{\costO{1}{1}} = \alpO{1}{n}\\
	\alpO{k}{n} &= \frac{\costO{k}{n}}{\costO{k}{k}} = \frac{\frac{k}{2}(1+\sqrt{5})}{\frac{k}{2}(\sqrt{5}-1)} =  \frac{1}{2}(3+\sqrt{5}) = \beta_3 > \frac{\costO{k}{1}}{\costO{k}{k}} = \alpO{k}{1}\\
	\alpO{n}{1} &= \frac{\costO{n}{1}}{\costO{n}{n}} = \frac{n}{\frac{n}{2}(3 - \sqrt{5})} = \frac{1}{2}(3+\sqrt{5}) = \beta_3 > \frac{\costO{n}{k}}{\costO{n}{n}} = \alpO{n}{k}
\end{aligned}
\]
This means that by picking any of the facility locations in $\F$, we will obtain a result that is at least as bad as $\beta_3$ approximation for one of the objectives, as desired.
\end{proof}
\label{prop:3ub}
\end{proposition}

\section{Choosing a Facility Location from Client Locations}
\label{section:same2}

So far in this paper we considered the general case when only some set of locations $\F$ allow the building of a facility. It is also interesting to consider the easier case, as was done in much of existing work \cite{chakrabarty2019approximation, alamdari2017bicriteria, chakrabarty2017interpolating}, when facilities are allowed to be built at any client location, i.e., when $\C = \F$. A lot of results become simpler with this assumption; for example it is easy to see that choosing any client location is immediately a 2-approximation for the $c_1$ (max distance) objective. Thus, choosing $O_n$ immediately gives a $(2,1)$ approximation for the $c_1$ (max) and $c_n$ (sum) objectives, while for general $\F$ we have shown that nothing better than $1+\sqrt{2}$ simultaneous approximation is possible. 
In fact, as an easy extension to the results shown in \cite{gkatzelis2020resolving} (using the fact that when $\C = \F$, {\em decisiveness} as defined in that paper equals zero, see Proposition 6 in the arXiv version of their paper), we can always get a 2 simultaneous approximation for arbitrarily many \emph{l-centrum} objectives. This is instead of the 3-approximation bound which we know holds for the general case when $\C \neq \F$.

In addition to the above result,  we can further extend our analysis from Section \ref{section:2_obj} to the case when $\C = \F$. For simultaneously approximating two objectives, combined with the result implied in \cite{gkatzelis2020resolving}, when $\C = \F$, the simultaneous approximation bound is the same as in Figure \ref{fig:2_ub} for $\frac{p}{k}\leq 4$, but for $\frac{p}{k}>4$ it simply equals 2. 
% check this ---> use the same examples as before (w/ graph metric for the 2 case)
It is not difficult to show that this bound is tight, using similar examples as our bounds from Section \ref{section:2_obj} for general facility locations $\F$.

%Note that if we know client locations $\C$, it is natural to consider the case where we want to choose a location to build a facility from client locations. 
%This means that we have $\C = \F$, $|\C| = |\F| = n$. 
%Now, consider two objectives $c_k$ and $c_p$ with $1 \leq p \leq n$. Denote the optimal facility location in $\F$ that minimize $c_k$ and $c_p$ We will first make the following observation:\\

% With the assumption that $\C = \F$, we want to see if we can improve our previous results for two objectives $c_k$ and $c_p$ by choosing some $A \in \F$ such that $\max(\alp{k}{A}, \alp{p}{A}) \leq 2$. 
% In fact, we can obtain the following results. 
% Note that unlike all the previous results in this paper, the outcome that is chosen is {\em not} one of the optimum outcomes for the objectives in $\K$; to obtain the best simultaneous approximation it is specifically necessary to choose an outcome which is sub-optimal for all individual objectives.
Note that unlike all the previous results in this paper, the outcome that is chosen may {\em not} one of the optimum outcomes for the objectives in $\K$; to obtain the best simultaneous approximation it is often necessary to choose an outcome which is sub-optimal for all individual objectives.

\section{Conclusion}
We have shown that, when selecting a facility according to multiple competing interests, it is always possible to form an outcome approximating several competing objectives, at least as long as these objectives are one of the $l$-centrum objectives. For instance, both minimizing the maximum cost and minimizing the total cost can be simultaneously approximated within a ratio of $1+\sqrt{2}$. We can in fact extend the obtained upper bound of the approximation ratio to a broader range of problems. As discussed in the Introduction, if we can get an $\alpha$ simultaneous approximation ratio for a set of \emph{l-centrum} objectives $\K$, then we can get an $\alpha$-approximation for the corresponding \emph{ordered 1-median} problems such that their objectives can be represented as convex combinations of the objectives in $\K$. This implies that there always exists a 3 approximation for all the \emph{ordered 1-median} problems, and in fact all convex combinations of any $q$ $\l$-centrum objectives can be simultaneously approximated within our ratio $\beta_q$, with for example $\beta_2=1+\sqrt{2}$.
% ordered norms (ordered k-median) are convex combination/non-negative linear combination of top-l norms (l-centrum)
% ---

However, there still exist questions left unanswered. For example, we do not know if the upper bound of the simultaneous approximation ratio for more than 3 objectives is tight, or if a better approximation is possible. 
%In addition, for the case when $\C = \F$, we only considered two objectives so far; it would also be worth exploring what would happen for more than two objectives in that case. 
More generally, it would be interesting to see if other types of objectives can be simultaneously approximated for these facility location and voting settings. 

\subsection*{Acknowledgements} This work was partially supported by NSF award CCF-2006286.

\bibliography{refs}

\begin{thebibliography}{10}

\bibitem{merrill1999unified}
Samuel Merrill~III, Samuel Merrill, and Bernard Grofman.
\newblock {\em A unified theory of voting: Directional and proximity spatial
  models}.
\newblock Cambridge University Press, 1999.

\bibitem{enelow1984spatial}
James~M. Enelow and Melvin~J. Hinich.
\newblock {\em The spatial theory of voting: An introduction}.
\newblock CUP Archive, 1984.

\bibitem{anshelevich2021distortion}
Elliot Anshelevich, Aris Filos-Ratsikas, Nisarg Shah, and Alexandros~A
  Voudouris.
\newblock Distortion in social choice problems: The first 15 years and beyond.
\newblock In {\em Proceedings of the Thirtieth International Joint Conference
  on Artificial Intelligence Survey Track.}, 2021.

\bibitem{slater1978centers}
Peter~J Slater.
\newblock Centers to centroids in graphs.
\newblock {\em Journal of graph theory}, 2(3):209--222, 1978.

\bibitem{tamir2001k}
Arie Tamir.
\newblock The k-centrum multi-facility location problem.
\newblock {\em Discrete Applied Mathematics}, 109(3):293--307, 2001.

\bibitem{peeters1998some}
Peter~H Peeters.
\newblock Some new algorithms for location problems on networks.
\newblock {\em European Journal of Operational Research}, 104(2):299--309,
  1998.

\bibitem{kumar2006fairness}
Amit Kumar and Jon Kleinberg.
\newblock Fairness measures for resource allocation.
\newblock {\em SIAM Journal on Computing}, 36(3):657--680, 2006.

\bibitem{goel2006simultaneous}
Ashish Goel and Adam Meyerson.
\newblock Simultaneous optimization via approximate majorization for concave
  profits or convex costs.
\newblock {\em Algorithmica}, 44(4):301--323, 2006.

\bibitem{goel2018relating}
Ashish Goel, Reyna Hulett, and Anilesh~K Krishnaswamy.
\newblock Relating metric distortion and fairness of social choice rules.
\newblock In {\em NetEcon'18: Proceedings of the 13th Workshop on Economics of
  Networks, Systems and Computation}, 2018.

\bibitem{bhargava2001using}
Rishi Bhargava, Ashish Goel, and Adam Meyerson.
\newblock Using approximate majorization to characterize protocol fairness.
\newblock {\em ACM SIGMETRICS Performance Evaluation Review}, 29(1):330--331,
  2001.

\bibitem{gkatzelis2020resolving}
Vasilis Gkatzelis, Daniel Halpern, and Nisarg Shah.
\newblock Resolving the optimal metric distortion conjecture.
\newblock In {\em Proceedings of the 61st {IEEE} Annual Symposium on
  Foundations of Computer Science ({FOCS})}, pages 1427--1438, 2020.

\bibitem{chan2021mechanism}
Hau Chan, Aris Filos-Ratsikas, Bo~Li, Minming Li, and Chenhao Wang.
\newblock Mechanism design for facility location problems: A survey.
\newblock In {\em Proceedings of the Thirtieth International Joint Conference
  on Artificial Intelligence, {IJCAI-21}}, pages 4356--4365. International
  Joint Conferences on Artificial Intelligence Organization, 2021.

\bibitem{farahani2010multiple}
Reza~Zanjirani Farahani, Maryam SteadieSeifi, and Nasrin Asgari.
\newblock Multiple criteria facility location problems: A survey.
\newblock {\em Applied mathematical modelling}, 34(7):1689--1709, 2010.

\bibitem{ehrgott2000survey}
Matthias Ehrgott and Xavier Gandibleux.
\newblock A survey and annotated bibliography of multiobjective combinatorial
  optimization.
\newblock {\em OR-spektrum}, 22(4):425--460, 2000.

\bibitem{alamdari2017bicriteria}
Soroush Alamdari and David Shmoys.
\newblock A bicriteria approximation algorithm for the k-center and k-median
  problems.
\newblock In {\em International Workshop on Approximation and Online
  Algorithms}, pages 66--75. Springer, 2017.

\bibitem{mcginnis1978single}
Leon~F McGinnis and John~A White.
\newblock A single facility rectilinear location problem with multiple
  criteria.
\newblock {\em Transportation Science}, 12(3):217--231, 1978.

\bibitem{ohsawa1999geometrical}
Yoshiaki Ohsawa.
\newblock A geometrical solution for quadratic bicriteria location models.
\newblock {\em European Journal of Operational Research}, 114(2):380--388,
  1999.

\bibitem{nickel2005multicriteria}
S~Nickel, J~Puerto, AM~Rodr{\'\i}guez-Ch{\'\i}a, and A~Weissler.
\newblock Multicriteria planar ordered median problems.
\newblock {\em Journal of Optimization Theory and Applications},
  126(3):657--683, 2005.

\bibitem{roostapour2015deterministic}
Vahid Roostapour, Iman Kiarazm, and Mansoor Davoodi.
\newblock Deterministic algorithm for 1-median 1-center two-objective
  optimization problem.
\newblock In {\em International Conference on Topics in Theoretical Computer
  Science}, pages 164--178. Springer, 2015.

\bibitem{sornat2019approximation}
Krzysztof Sornat.
\newblock {\em Approximation Algorithms for Multiwinner Elections and
  Clustering Problems}.
\newblock PhD thesis, University of Wrocław, 2019.

\bibitem{chakrabarty2017interpolating}
Deeparnab Chakrabarty and Chaitanya Swamy.
\newblock Interpolating between k-median and k-center: Approximation algorithms
  for ordered k-median.
\newblock In {\em 45th International Colloquium on Automata, Languages, and
  Programming (ICALP 2018)}. Schloss Dagstuhl-Leibniz-Zentrum fuer Informatik,
  2018.

\bibitem{chakrabarty2019approximation}
Deeparnab Chakrabarty and Chaitanya Swamy.
\newblock Approximation algorithms for minimum norm and ordered optimization
  problems.
\newblock In {\em Proceedings of the 51st Annual ACM SIGACT Symposium on Theory
  of Computing}, pages 126--137, 2019.

\bibitem{byrka2018constant}
Jaros{\l}aw Byrka, Krzysztof Sornat, and Joachim Spoerhase.
\newblock Constant-factor approximation for ordered k-median.
\newblock In {\em Proceedings of the 50th Annual ACM SIGACT Symposium on Theory
  of Computing}, pages 620--631, 2018.

\bibitem{kalcsics2002algorithmic}
J{\"o}rg Kalcsics, Stefan Nickel, Justo Puerto, and Arie Tamir.
\newblock Algorithmic results for ordered median problems.
\newblock {\em Operations Research Letters}, 30(3):149--158, 2002.

\bibitem{walsh2021strategy}
Toby Walsh.
\newblock Strategy proof mechanisms for facility location at limited locations.
\newblock In {\em Pacific Rim International Conference on Artificial
  Intelligence}, pages 113--124. Springer, 2021.

\bibitem{alon2010strategyproof}
Noga Alon, Michal Feldman, Ariel~D Procaccia, and Moshe Tennenholtz.
\newblock Strategyproof approximation of the minimax on networks.
\newblock {\em Mathematics of Operations Research}, 35(3):513--526, 2010.

\bibitem{feldman2016voting}
Michal Feldman, Amos Fiat, and Iddan Golomb.
\newblock On voting and facility location.
\newblock In {\em Proceedings of the 2016 ACM Conference on Economics and
  Computation}, pages 269--286, 2016.

\bibitem{tang2020mechanism}
Zhongzheng Tang, Chenhao Wang, Mengqi Zhang, and Yingchao Zhao.
\newblock Mechanism design for facility location games with candidate
  locations.
\newblock In {\em International Conference on Combinatorial Optimization and
  Applications}, pages 440--452. Springer, 2020.

\bibitem{aziz2021strategyproof}
Haris Aziz, Alexander Lam, Barton~E Lee, and Toby Walsh.
\newblock Strategyproof and proportionally fair facility location.
\newblock {\em arXiv preprint arXiv:2111.01566}, 2021.

\bibitem{anshelevich2017randomized}
Elliot Anshelevich and John Postl.
\newblock Randomized social choice functions under metric preferences.
\newblock {\em Journal of Artificial Intelligence Research}, 58:797--827, 2017.

\end{thebibliography}

\end{document}